\documentclass[twoside]{article}

\usepackage[accepted]{aistats2022}

\usepackage[utf8]{inputenc} %
\usepackage[T1]{fontenc}    %
\usepackage{hyperref}       %
\usepackage{url}            %
\usepackage{booktabs}       %
\usepackage{amsfonts}       %
\usepackage{nicefrac}       %
\usepackage{microtype}      %
\usepackage{caption}
\usepackage{amssymb}
\usepackage{graphicx}
\usepackage{color}
\usepackage{complexity}
\usepackage{subcaption}
\usepackage{amsmath}
\usepackage{amsthm}
\usepackage{pifont}
\usepackage{tabularx}
\usepackage{fancyvrb}
\usepackage{comment}
\usepackage{algorithm}
\usepackage{algorithmic}
\usepackage{float}
\usepackage{wrapfig}
\usepackage{mathtools}
\usepackage{thmtools,thm-restate}
\usepackage{paralist} 
\usepackage{multirow}
\usepackage{scrextend}
\usepackage{enumitem}
\usepackage[noorphans,vskip=0ex]{quoting}
\usepackage{bm}
\usepackage{bbm}
\usepackage{thm-restate}
\usepackage{tabularx}
\usepackage{xcolor}
\usepackage{xspace}
\usepackage{adjustbox}
\usepackage{textcomp}
\usepackage{pgf,tikz,pgfplots}
\pgfplotsset{compat=1.15}
\usepackage{mathrsfs}
\usetikzlibrary{arrows}
\usepackage[round]{natbib}

\setlist{nolistsep}
\theoremstyle{definition}
\newtheorem{definition}{Definition}[section]

\newcommand{\aaspace}{\mathcal{A}}

\newcommand{\sspace}{\mathcal{S}}

\newcommand{\RR}{\mathbb{R}}

\newcommand{\Exp}{\mathbb{E}}

\newcommand*\dif{\mathop{}\!\mathrm{d}}

\newcommand\mathbi[1]{\textbf{\em #1}}

\DeclareMathOperator*{\argmax}{arg\,max}

\newcommand{\defeq}{\vcentcolon=}

\newcommand{\mmd}{\mathrm{MMD}}
\DeclarePairedDelimiterX{\inp}[2]{\langle}{\rangle}{#1, #2}

\newcommand{\retgap}{\Delta_{\mathrm{Ret}}}

\newcommand{\ie}{\textit{i}.\textit{e}.}
\newcommand{\eg}{\textit{e}.\textit{g}.}

\begin{document}

\runningtitle{Towards Return Parity in Markov Decision Processes}

\runningauthor{J. Chi,~~J. Shen,~~D. Dai,~~W. Zhang,~~Y. Tian,~and~~H. Zhao}

\twocolumn[

\aistatstitle{Towards Return Parity in Markov Decision Processes}

\aistatsauthor{Jianfeng Chi \And Jian Shen \And Xinyi Dai} 

\aistatsaddress{University of Virginia \And Shanghai Jiao Tong University \And Shanghai Jiao Tong University}

\aistatsauthor{Weinan Zhang \And Yuan Tian \And Han Zhao}
\aistatsaddress{Shanghai Jiao Tong University \And University of Virginia \And University of Illinois at Urbana-Champaign}
]

\begin{abstract}
  Algorithmic decisions made by machine learning models in high-stakes domains may have lasting impacts over time. However, naive applications of standard fairness criterion in static settings over temporal domains may lead to delayed and adverse effects. To understand the dynamics of performance disparity, we study a fairness problem in Markov decision processes (MDPs). Specifically, we propose return parity, a fairness notion that requires MDPs from different demographic groups that share the same state and action spaces to achieve approximately the same expected time-discounted rewards.
  We first provide a decomposition theorem for return disparity, which decomposes the return disparity of any two MDPs sharing the same state and action spaces into the distance between group-wise reward functions, the discrepancy of group policies, and the discrepancy between state visitation distributions induced by the group policies.
  Motivated by our decomposition theorem, we propose algorithms to mitigate return disparity via learning a shared group policy with state visitation distributional alignment using integral probability metrics.
  We conduct experiments to corroborate our results, showing that the proposed algorithm can successfully close the disparity gap while maintaining the performance of policies on two real-world recommender system benchmark datasets.
\end{abstract}

\section{Introduction}

The increasing use of automated decision-making systems trained with real-world data has raised serious concerns with potential unfairness caused by biased data, learning algorithms, and models.
Decisions made by these systems have lasting and diverse effects on different social groups. 
For example, in predictive policing~\citep{lum2016predict}, each time, the decisions about the locations of future crimes are made by the predictive models. As a result, the discovered crime rates in different communities might change dynamically and interactively as feedback to the decisions being made by the models. Thus, the error rates of the predictive models for different communities might change over time, and error gaps among communities could possibly exacerbate in the long run. Similar feedback loops could also exist in the applications such as recommender systems (\eg, the user satisfactions from different demographic groups diverge over time), temporal resource allocation systems (\eg, the uneven resource allocation become more skew towards one group over the others), etc. 
This interplay between algorithmic decisions and the heterogeneous reactions caused by the decisions further complicates the analysis of (un)fairness problems in the dynamic environment. 

Most prior works mainly focus on studying static fairness notions (\eg, demographic parity~\citep{dwork2012fairness} and equalized odds~\citep{hardt2016equality}) in the settings of classification~\citep{hardt2016equality,zafar2017fairness,zhao2019inherent,jiang2020wasserstein} and regression~\citep{komiyama2018nonconvex, agarwal2019fair, chzhen2020fair, chi2021understanding}. 
In a seminal work, \citet{liu2018delayed} show that somewhat contrary to the common belief, enforcing static fairness constraints could do harm to the minority group even in a one-step feedback model. Motivated by this observation, a line of work aims to extend static fairness notions in the settings of online learning~\citep{blum2018preserving, bechavod2019equal} and multi-armed bandit~\citep{joseph2016fairness, patil2020achieving, chen2020fair}. However, these works do not take into account the interactions between the models and the environment: the decisions made by the models could potentially influence the state of our environment as well, as demonstrated by the predictive policing example. Other works study the interplay between (static) fairness notions and the population dynamics under simplified temporal dynamic models~\citep{hu2018short, ensign2018runaway, pmlr-v80-hashimoto18a, mouzannar2019fair, zhang2019group, elzayn2019fair, liu2020disparate}.
However, those simplified temporal dynamic models explicitly make task-specific assumptions on temporal dynamics and might not be able to precisely characterize the complex dynamics of a more general changing environment.

In this work, we study a fairness problem in Markov decision processes (MDPs) to understand the dynamics of performance disparity in the changing environments, taking into account the feedback loop caused by policies to the environments. Specifically, we propose \emph{return parity}, a novel fairness notion that requires MDPs from different demographic groups that share the same state and action spaces to achieve approximately the same expected time-discounted rewards. To the best of our knowledge, our work is the first of this kind, in the sense that we study the long-term impact of policy functions in general MDPs. First, we formally show exact return parity cannot always be satisfied for any two MDPs and provide sufficient conditions that ensure return parity in terms of transitions, initial distributions, and the reward functions. 
Next, we derive a decomposition theorem for return disparity that decomposes it into the distance between group-wise reward functions, the discrepancy of group policies, and the discrepancy between state visitation distributions induced by the group policies. Motivated by the decomposition theorem, we propose algorithms to mitigate return disparity via learning a shared group policy with state visitation distributional alignment using integral probability metrics.
We conduct experiments on two real-world benchmark datasets to corroborate the effectiveness of our proposed algorithms in reducing return disparity. Experimental results demonstrate that our proposed algorithms help to mitigate return disparity while maintaining the performance of policies. 

\section{Preliminaries}

\paragraph{Notation and Problem Setup} 
Throughout the paper, we mainly focus on discrete MDPs, where both the state and action spaces are finite.\footnote{The main results in Section~\ref{sec:main-results} could be extended to continuous state space and action space.} A Markov decision process is a tuple $(\sspace,\aaspace, \mu, T, r,\gamma)$, where $\sspace$ and $\aaspace$ are the state space and the action space, respectively;
$\mu$ is initial state distribution; $T\in \RR^{|\sspace| \times |\aaspace| \times |\sspace|} $ is the state transition function where $T(s' \mid s, a)=\Pr[S_{t+1}=s'\mid S_t=s, A_t=a]$; %
$r\in \RR^{ |\sspace| \times |\aaspace| }$ is the reward function where $r(s, a)$ represents the reward obtained when taking action $a$ in state $s$. Throughout the paper, we assume that the reward function is uniformly bounded, \ie, $\exists~ R > 0,~\|r\|_\infty\leq R$. We use $\gamma\in(0, 1)$ to denote the discount factor.
Given a policy
$\pi\in\RR^{|\sspace| \times |\aaspace|}$ (\ie, $\pi(a \mid s) = \Pr[A_t=a\mid S_t=s]$),
the induced state transition under $\pi$ is $P^{\pi}\in \RR^{|\sspace| \times |\sspace|} $ where $P^{\pi}(s'\mid s)= \sum_{a\in\aaspace}\pi(a \mid s) T(s' \mid s, a)$. The induced distribution over states under the policy $\pi$ at time step $t$ is 
\begin{equation}
\nonumber
\mu^{(\pi,t)} = \begin{cases}
\mu &\text{if $t=0$}\\
P^{\pi}\mu^{(\pi,t-1)} &\text{otherwise.}
\end{cases}
\end{equation}

The state visitation distribution (\ie, time-discounted distribution over states) is
$\mu^\pi = (1-\gamma) \sum_{t=0}^\infty \gamma^t \mu^{(\pi,t)} = (1-\gamma) \sum_{t=0}^\infty (\gamma P^{\pi})^t \mu$ and the occupancy measure (\ie, time-discounted distribution over state-action pairs) is $\rho^\pi(s, a)= \mu^\pi(s) \pi(a \mid s)$. We then define the value function w.r.t. the reward function $r$ under the policy $\pi$ as $v^\pi(s) = \Exp_\pi[\sum_{t=0}^\infty\gamma^{t}r_t\mid S_0=s]$ and the Q-function as $q^\pi(s, a) = \Exp_\pi[\sum_{t=0}^\infty\gamma^{t}r_t\mid S_0=s, A_0=a]$,
where $r_t$ is the immediate reward at time step $t$.
With all the notation defined above, the goal of reinforcement learning is to find a policy to maximize the value (expected return) under the initial state distribution:
\begin{equation}
    \nonumber
    \centering
    \begin{aligned}
    \eta^{\pi} \defeq 
    \Exp_{s\in\mu}[v^\pi(s)] &= \frac{1}{1-\gamma}
    \sum_{s\in\sspace}\sum_{a\in\aaspace} r(s, a) \rho^\pi(s, a)\\
    &= \frac{1}{1-\gamma} \sum_{s\in\sspace} \sum_{a\in\aaspace} r(s, a) \mu^\pi(s) \pi(a \mid s).
    \end{aligned}
\end{equation}
Let $|\sspace| = m$ and $|\aaspace| = n$. In practice, each state $s\in\sspace$ might represent features of an individual and the action enforced on the individual could lead to the change of features of the individual. 
We also assume there are two Markov decision processes that represent two different demographic groups (\eg, male/female, white/non-white) and the two MDPs share the same state space, action space, and discount factor but might differ in initial distributions, transitions, and reward functions. We use the subscript $g\in\{0, 1\}$ to denote the two groups. For example, $\mu_0$ and $\mu_1$ are the initial distributions of the two groups, respectively. Note that in our paper, we mainly discuss the setting there are two different demographic groups and follow the standard definition of return in the time-discounted MDPs, but the underlying theory and algorithms could easily be extended to the cases with finite $K > 2$ groups and undiscounted MDPs with finite time-horizon.

Next, we define \emph{$\epsilon$-return parity} as a fairness criterion to ask that different demographic groups share approximately the same long-term rewards under a policy:
\begin{definition}[$\epsilon$-Return Parity]
For $g\in\{0, 1\}$,  two MDPs satisfy $\epsilon$-\emph{return parity} if $\retgap \defeq | \Exp_{s\in\mu_0}[v^{\pi_0}(s)] - \Exp_{s\in\mu_1}[v^{\pi_1}(s)] | \leq \epsilon$.
\end{definition}

Return parity could have different implications depending on the scenarios we consider: In recommender systems, if reward function corresponds to be users' satisfaction, return parity seeks similar users' satisfaction across different demographic groups in the long run; In predictive policing, if we define the reward function as the ratio between truly discovered incidents of crime (\ie, those directly observed by dispatched police as a result of the predictive policing algorithm) and the overall predicted incidents of crime in a time period, return parity requires a similar ratio for different communities over time.
The complex temporal dynamic of the above scenarios could be modeled by MDPs.
The goal in our setting is then to find two policies that maximize the weighted combination of expected returns of two MDPs respectively while satisfying $\epsilon$-return parity:
\begin{equation}
    \nonumber
    \begin{aligned}
     \max_{\pi_0, \pi_1}\quad& \lambda\, \eta^{\pi_0} + (1-\lambda)\, \eta^{\pi_1}, \qquad\textrm{s.t.}\quad \retgap \leq \epsilon,\\
    \end{aligned}
\end{equation}
where $\lambda \in [0, 1]$ represents the proportion of group 0 over the entire population.

\paragraph{Integral Probability Metrics} The integral probability metrics (IPMs) are a class of distance measures on probability distributions over the same probability space~\citep{muller1997integral}. Formally, given two probability distributions $\mathcal{P}$ and $\mathcal{Q}$, the IPMs are defined as $d_{\mathcal{F}}(\mathcal{P}, \mathcal{Q}) = \sup_{f\in\mathcal{F}} |\int f \dif  \mathcal{P} - \int f \dif  \mathcal{Q}|$, where $\mathcal{F}$ is a class of real-valued bounded measurable
functions on the space where the distributions are defined on. Different choices of $\mathcal{F}$ recover different distance metrics: 
If we choose $\mathcal{F} = \{ f: \| f \|_L \leq 1 \}$ where $\|\cdot\|_L$  denotes the Lipschitz semi-norm, then $d_{\mathcal{F}}(\mathcal{P}, \mathcal{Q})$ becomes \emph{Wasserstein-1 distance} $W_1(\mathcal{P}, \mathcal{Q})$; If we choose $\mathcal{F} = \{ f: \| f \|_\mathcal{H} \leq 1 \}$ where $\|\cdot\|_\mathcal{H}$ denotes the norm in a reproducing kernel Hilbert space (RKHS), then $d_{\mathcal{F}}(\mathcal{P}, \mathcal{Q})$ becomes \emph{maximum mean discrepancy} $\mmd(\mathcal{P}, \mathcal{Q})$. 

\section{Analysis of Return Parity}
\label{sec:main-results}

In this section, we first show that return parity cannot always be satisfied between two MDPs that share the same state and action spaces and provide sufficient conditions under which return parity is possible. Then, we provide more insights into return disparity by proving an upper bound of the return gap between two MDPs.
We defer all the detailed proofs to Appendix~\ref{app-sec:proof} due to the space limit.

\subsection{Is Return Parity Always Possible?}

Before we provide a rigorous analysis of return disparity in MDPs, it is vital to ask whether the exact return parity is always achievable. The following proposition gives a negative answer to this question:
\begin{restatable}{proposition}{ImpossiblityResults}
For any constant $c > 0$, there exist two MDPs that share the same state and action spaces, such that $\forall~ \pi_0, \pi_1 \in \Pi$, the return disparity $\retgap \geq c$.
\label{prop:imp}
\end{restatable}

For example, consider two MDPs share two states $s_1$ and $s_2$. Let $r(s_1, a)= c (1-\gamma) > 0$ and $r(s_2, a) = 0,~\forall a\in\aaspace$, $T(s_2\mid s_1, a)=T(s_1\mid s_2, a) = 0$ and $T(s_1\mid s_1, a)=T(s_2\mid s_2, a) = 1,~\forall a\in\aaspace$. Given $\mu_0 = [1, 0]^T$ and $\mu_1 = [0, 1]^T$, then the return gap $\retgap = c > 0$. In this case, it is impossible to find policies to satisfy $\epsilon$-return parity for any $\epsilon < c$. 

In addition, it is also natural to ask whether it is feasible to maximize the expected returns of the two MDPs simultaneously while achieving $\epsilon$-return parity in general. Formally, with the help of the linear programming (LP) approach for MDPs~\citep{de2002linear, de2003linear} and the duality of LP, it is equivalent to solve the following dual LP:
\begin{equation}
    \small
    \nonumber
    \begin{aligned}
    \max& \sum_s \sum_a \hat{\rho}_0 (s, a) r_0(s, a) + \hat{\rho}_1 (s, a) r_1(s, a) - \epsilon (b_0 + b_1) \\
    \text{s.t.}& \sum_a \hat{\rho}_0 (s_i, a) -\gamma \sum_{s} \sum_{a} T_0(s_i \mid s, a) \hat{\rho}_0 (s, a) \\[-0.5ex]
    & ~~~~~~~~~~~~~~~~~~~~~~~~~~= (\lambda + b_0 - b_1) (\mu_0)_i~~~~~~~\forall~i\in[m]  \\
    & \sum_a \hat{\rho}_1 (s_i, a) -\gamma \sum_{s} \sum_{a} T_1(s_i \mid s, a) \hat{\rho}_1 (s, a) \\[-0.5ex]
    & ~~~~~~~~~~~~~~~~~~~~~~~~~~= (1-\lambda + b_1 - b_0)(\mu_1)_i~~\forall~i\in[m], \\
    \end{aligned}
\end{equation}
where $\hat{\rho}_0 (s, a), \hat{\rho}_1 (s, a), b_0, b_1\geq 0,~\forall~s, a$ are dual variables. 
Note that $\hat{\rho}_0 (s, a)$ and $\hat{\rho}_1 (s, a)$ have the interpretation of discounted state-action counts of the policy when $b_0 = b_1$, and $b_0, b_1$ are the corresponding dual variables of the $\epsilon$-return parity constraint, representing the ``per unit cost’’ of the overall return to achieve $\epsilon$-return parity. The dual constraints are state transitions under the learned policies.
We can now characterize a sufficient condition for the optimal policies $\pi_0$ and $\pi_1$ to satisfy $\epsilon$-return parity with the dual formulation above.
\begin{restatable}{proposition}{PossiblityResults}
For $\forall~\hat{\rho}_0 (s, a), \hat{\rho}_1 (s, a), b_0, b_1\geq 0$, if there exists $i \in [m]$, such that
\begin{equation}
    \nonumber
    \small
    \begin{aligned}
        \sum_a \hat{\rho}_0(s_i, a) - \gamma \sum_{s} \sum_{a} T_0(s_i \mid s, a) \hat{\rho}_0 (s, a) &> (b_0 - b_1) (\mu_0)_i \\
        \sum_a \hat{\rho}_1 (s_i, a) -\gamma \sum_{s} \sum_{a} T_1(s_i \mid s, a) \hat{\rho}_1 (s, a) &> (b_1 - b_0) (\mu_1)_i  \\
        \sum_s \sum_a \hat{\rho}_0 (s, a) r_0(s, a) + \hat{\rho}_1 (s, a) r_1(s, a) &\leq (b_0+b_1)\epsilon \\
    \end{aligned}
\end{equation}
then the optimal policies $\pi_0^*$ and $\pi_1^*$ that maximize the expected returns of two MDPs satisfy $\epsilon$-return parity.
\label{prop:suff_v2}
\end{restatable}

In light of the dual LP formulation, the first two inequalities in Proposition~\ref{prop:suff_v2} indicate the probability masses of the initial distributions in at least one state are greater than zero, and the last inequality requires the maximum value of the objective function in the dual LP formulation is no greater than zero.

\subsection{A Decomposition Theorem for Return Disparity}

The linear programming methods to solve the return disparity problem of MDPs become impractical in continuous or high-dimensional discrete state and action spaces. 
However, in many real-world scenarios where return parity is desired, the number of the states or actions is often large (\eg, recommend items to different demographic groups of users).
In this section, we shall provide an upper bound to (1) quantitatively characterize return disparity in terms of the distance between the reward functions, the discrepancy of group policies, and the discrepancy between state visitation distributions and, (2) motivate our algorithm design to mitigate return disparity in continuous or high-dimensional discrete state and action spaces (Sec.~\ref{sec:alg}). 

\begin{restatable}{theorem}{rewardDisparityUpperBoundState}
For $g\in \{0, 1\}$, given policies $\pi_0, \pi_1\in \Pi$ and assume there exists a witness function class $\mathcal{F} = \{f: \sspace \times \aaspace \to \RR \}$, such that the reward functions $r_g(s) = \Exp_{a\sim \pi_g(\cdot \mid s)}[ r_g(a, s) \mid s]\in\mathcal{F}$ for $\forall~s\in\sspace$, $a\in\aaspace$, and $g\in\{0, 1\}$, then the following holds: 
\begin{equation}
    \centering
    \nonumber
    \begin{aligned}
     \retgap \leq \frac{1}{1-\gamma}&\bigg(\|r_0 - r_1\|_\infty + R\cdot\min \big\{ \Exp_{s\sim\mu^{\pi_0}}\big[ d_{\pi_0, \pi_1}(s) \big],\\
     & \Exp_{s\sim\mu^{\pi_1}}\big[ d_{\pi_0, \pi_1}(s) \big] \big\} + d_{\mathcal{F}} (\mu^{\pi_0}(s), \mu^{\pi_1}(s)) \bigg),\\
    \end{aligned}
\end{equation}
where $d_{\pi_0, \pi_1}(s) \defeq \|\pi_0(\cdot\mid s) - \pi_1(\cdot\mid s)\|_1$. 
\label{thm:main}
\end{restatable}
\paragraph{Remark} We see that return disparity is upper bounded by three terms: the distance between group-wise reward functions, the discrepancy of group policies, and the discrepancy between state visitation distributions of the two MDPs. Given any two MDPs, the distance between group-wise reward functions is constant. It suggests that if the reward functions from two groups are drastically different, it may not be possible to ensure return parity by only looking at the policies and the state-visitation distributions. If we further assume the two MDPs share the same reward functions (\ie, $r_0 = r_1$) and the same policy (\ie, $\pi_0 = \pi_1$), then the upper bound in Theorem~\ref{thm:main} is simplified as
\begin{equation}
    \nonumber
    \retgap \leq d_{\mathcal{F}}\big( \mu^{\pi_0}(s),  \mu^{\pi_1}(s) \big).
\end{equation}
In this case, it implies a sufficient condition to minimize return disparity is to find policies $\pi \in \Pi$ that minimize the distance between induced state visitation distributions in the two MDPs. In the subsequent sections, we assume the reward functions of different groups are (approximately) the same and propose algorithmic interventions to mitigate return disparity.
For completeness, we also provide another decomposition theorem (Theorem~\ref{thm:occ}) in Appendix~\ref{app-sec:thm} and motivate the design of another family of algorithms based on Theorem~\ref{thm:occ} in Appendix~\ref{app-sec:alg}. We present the Theorem~\ref{thm:main} in the main text since the algorithms (see Sec.~\ref{sec:alg}) motivated by it are more efficient and stable in the application scenarios (\eg, recommender system) we consider.

\section{Mitigation of Return Disparity}
\label{sec:alg}
Inspired by the theoretical results in Theorem~\ref{thm:main}, we introduce a state visitation distribution alignment procedure, which can be naturally incorporated into existing RL methods, to encourage policies to maximize expected returns while keeping state visitation distribution similar to each other. We use deep Q-networks~\citep{mnih2015human} as our baseline backbone algorithm, which has demonstrated its superior performance in recommender application~\citep{zheng2018drn}, in which we will conduct experiments next. In what follows, we first briefly introduce how to learn the deep Q-networks and then discuss state visitation distributional alignment via IPMs. We give the main pipeline of our algorithms in Algorithm~\ref{alg}.

\subsection{Preliminaries: Learning Deep Q-networks}

The main idea behind learning deep Q-networks (Q-learning) is to approximate the value functions in high dimensional state and action spaces. Specifically, we aim at learning deep Q-network $Q(s, a; \theta): \sspace \times \aaspace \to \RR$ to approximate the reward when taking an action in a given state. Given the deep Q-network, we can construct the policy that maximizes the rewards: $\pi(s) = \argmax_a Q(s, a)$. When the action space is discrete such as recommendation applications, the Q-network is often implemented as $Q(s, a; \theta): \sspace \to \RR^{|\aaspace|}$ for efficiency, where the value in output dimension $i$ represents estimated reward when taking action $a_i$ given the state $s$.

During model training, the parameter $\theta$ of the Q-network is trained through a trial-and-error process. Take interactive recommendation process as an example: at each time step $t$, the recommender agent obtains a user’s state $s_t$, and takes an action $a_t$ (\eg, recommend an item) via the $\varepsilon$-greedy policy (\ie, with probability $1-\varepsilon$ taking an action with the max Q-value, and with probability $\varepsilon$ choosing a random action). The agent then receives the reward $r_t$ (\eg, rating score on the recommended item) and the updated state $s_{t+1}$ from the user's feedback and stores the experience $(s_t, a_t, r_t, s_{t+1})$ in replay buffer $\mathbi{D}$. After updating the replay buffer with batches of experiences from different users, the agent then optimizes the following loss function to improve the Q-network:
\begin{equation}
    \mathcal{L}_{Q}(\theta) = \Exp_{(s_t, a_t, r_t, s_{t+1})\sim\mathbi{D}}[(y_t - Q(s, a; \theta))^2],
    \label{eq:loss-q}
\end{equation}
with 
\begin{equation}
    y_t = r_t + \gamma \max_{a_{t+1}} Q(s_{t+1}, a_{t+1}; \theta).
    \label{eq:orginal-y}
\end{equation}

To avoid the overestimation problem~\citep{thrun1993issues} in original DQN, we adopt the double DQN architecture~\citep{van2016deep}: a target network $Q'$ parameterized by $\theta'$ is utilized along with the online network $Q$. The online network is updated at each model update step, and the target network is a duplicate of the online network and updated with delay (soft update):
\begin{equation}
    \theta' = \tau \theta + (1-\tau) \theta',
    \label{eq:target-q-update}
\end{equation}
where $\tau$ controls the update frequency.
With the double DQN architecture, $y_t$ in Eq.~\ref{eq:orginal-y} is changed to
\begin{equation}
    y_t = r_t + \gamma \max_{a_{t+1}} Q'(s_{t+1}, \argmax_{a_{t+1}}Q(s_{t+1}, a_{t+1}; \theta) ; \theta').
    \label{eq:y}
\end{equation}

\begin{algorithm}[t]
 			\caption{Algorithm to mitigate return disparity under the framework of DQN.}
			\label{alg}
			\begin{algorithmic}[1]
				\STATE Initialize Q-functions $Q_{\theta}$ and $Q'_{\theta'}$, feature extractors $f_{\psi_0}$ and $f_{\psi_1}$, environment buffers $\mathbi{D}_0$ and $\mathbi{D}_1$
				\FOR{each iteration}
			        \FOR{each environment step}
			            \FOR{$g\in\{0, 1\}$}
			            \STATE Sample an action $a_g$ using $\varepsilon$-greedy policy; Add the experience $(s_g, a_g, s^\prime_g, r_g)$ to $\mathbi{D}_g$;
			            \ENDFOR
			        \ENDFOR
				    \FOR{each model update step}
				        \STATE Sample a mini-batch of experiences from $\mathbi{D}_0$ and $\mathbi{D}_1$; Update online and target DQNs and feature extractors $f_{\psi_g},~g\in\{0, 1\}$ using Eq.~\ref{eq:loss-q} and Eq.~\ref{eq:target-q-update};
			        \ENDFOR
			        \FOR{each state visitation distributional alignment step}
			        \STATE Update feature extractors $f_{\psi_0}$ and $f_{\psi_1}$ using Eq.~\ref{eq:minimax} or Eq.~\ref{eq:mmd};
			        \ENDFOR
				\ENDFOR
			\end{algorithmic}
\end{algorithm}

\subsection{State Visitation distributional Alignment}

Inspired by Theorem~\ref{thm:main}, we propose to align the state visitation distribution via learning group-invariant representation. 
Specifically, we introduce two feature extractors $f_{\psi_g},~g\in\{0, 1\}$ before inputting $s_t$ to Q-function. We alternatively optimize $f_{\psi_g}$ between minimizing the loss of Q-function (Eq.~\ref{eq:loss-q}) and minimizing the loss of state visitation distributional alignment. At each iteration, the feature extractors and the Q-function, parameterized by $\theta' = \{\psi_g \circ \theta\}$, are first trained jointly to minimize $\mathcal{L}_q(\theta')$. After updating the models, the feature extractors are then updated to minimize the loss of state visitation distributional alignment via minimizing the estimated integral probability metrics between different groups. The whole algorithm is shown in Algorithm~\ref{alg}. Next, we introduce the detailed steps of state visitation distributional alignment via Wasserstein-1 distance, which is a kind of IPM choosing $\mathcal{F} = \{ f: \| f \|_L \leq 1 \}$ and is more favorable than total variation distance from the optimization perspective~\citep{arjovsky2017wasserstein}.

\paragraph{Wasserstein-1 Distance}

Given samples $s_0$ from the state visitation distribution of group 0 and samples $s_1$ from the state visitation distribution of group 1, two feature extractors $f_0$ and $f_1$, with parameters $\psi_0$ and $\psi_1$, map the samples to feature representations: $h_0 = f_0(s_0)$ and $h_1 = f_1(s_1)$. 
Upon receiving the feature representations from both groups, we use a critic function $f_c$, parameterized by $\omega$, to estimate Wasserstein-1 distance~\citep{arjovsky2017wasserstein} by maximizing the following objective function:
\begin{equation}
    \mathcal{L}_{\text{Wass}}(\psi_0, \psi_1, \omega) = \frac{1}{N_0} \sum_{i=1}^{N_0} f_c(h_0^i) - \frac{1}{N_1} \sum_{j=1}^{N_1} f_c(h_1^j).
\end{equation}

To enforce the Lipschitz constraint on the critic function $f_c$, we choose to minimize the gradient penalty loss~\citep{gulrajani2017Improved} meanwhile:
\begin{equation}
    \mathcal{L}_{\text{GP}}(\omega) = \Exp_{\hat{h}\sim \hat{\mathbi{D}}}[( \Vert \nabla f_c(\hat{h}) \Vert_2 -1)^2],
\end{equation}

where $\hat{\mathbi{D}}$ represents the distribution of a uniformly distributed linear interpolations between the group visitation distributions. Finally, the overall minimax objective becomes
\begin{equation}
    \min_{\psi_0, \psi_1} \max_{\omega} \mathcal{L}_{\text{Wass}}(\psi_0, \psi_1, \omega) - \beta  \mathcal{L}_{\text{GP}}(\omega), 
    \label{eq:minimax}
\end{equation}
where $\beta$ is the balancing coefficient and by convention it is set to be 10~\citep{gulrajani2017Improved}. 
To ensure the training stability and do less harm to the group that obtains lower return, we use block coordinate descent algorithm~\citep{wright2015coordinate} to update the feature extractors: at each iteration of state visitation distributional alignment, we first identify the group with higher return choose to update its feature extractor while fixing the parameters of the feature extractor in the other group. Figure~\ref{fig:algo-framework} illustrates our algorithm framework. 

\begin{figure}[!t]
    \centering
    \includegraphics[width=\linewidth]{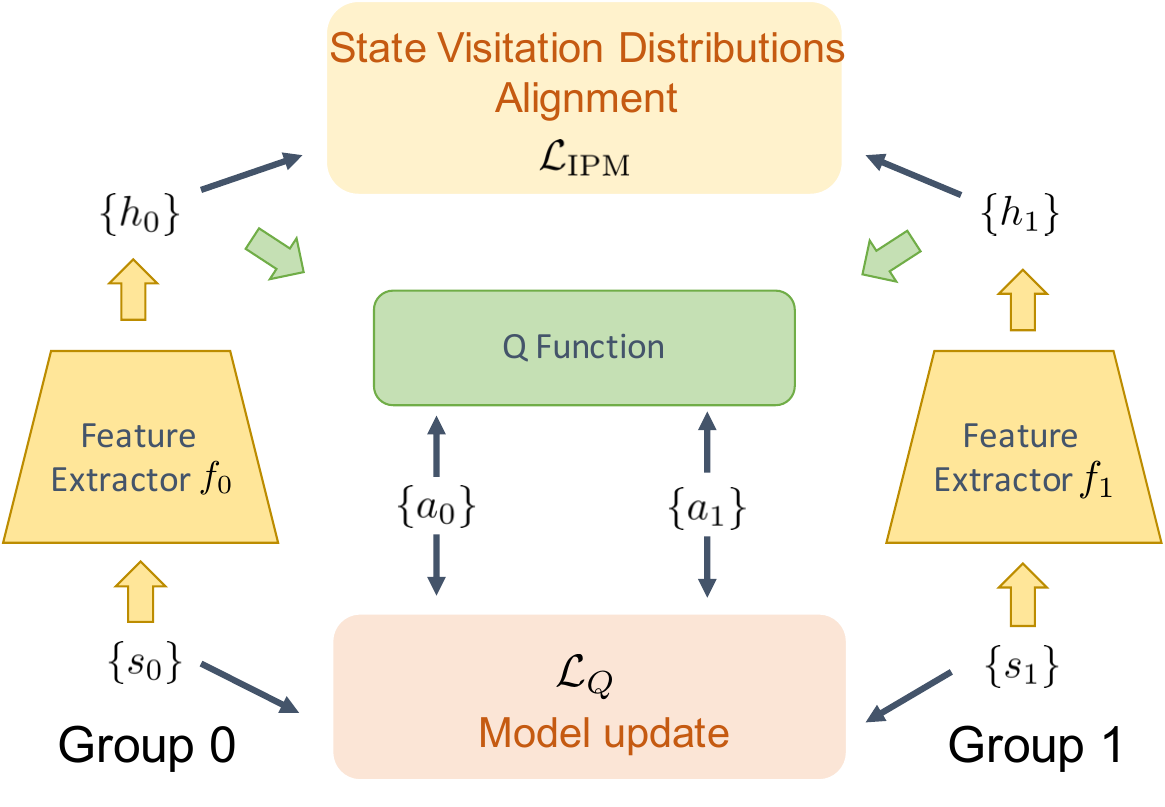}
    \vspace{-1em}
    \caption{Illustration of our algorithms.
    At every iteration, we first update the models (\eg, Q-function and feature extractors). The feature extractors are then updated to minimize the loss of state visitation distributional alignment via minimizing the estimated integral probability metrics between different groups.
    }
    \label{fig:algo-framework}
    \vspace{-1em}
\end{figure}

\paragraph{MMD Variant}

Similarly to Wasserstein distance, we can also use maximum mean discrepancy to align state visitation distributions. Let $k$ be the kernel of the corresponding RKHS $\mathcal{H}$ on the feature space, then the squared MMD between the induced feature distributions of two groups $\mathbi{D}_{h_0}$ and $\mathbi{D}_{h_1}$ is 
\begin{equation}
    \nonumber
    \begin{aligned}
    \text{MMD}^2(\mathbi{D}_{h_0}, \mathbi{D}_{h_1}) \defeq& \Exp_{h_0, h_0'}[k(h_0, h_0')] + \Exp_{h_1, h_1'}[k(h_1, h_1')]\\
    &- 2 \Exp_{h_0, h_1}[k(h_0, h_1)].
    \end{aligned}
\end{equation}

In practice, given samples from $\{h_0^1, \dots, h_0^{N_0}\} \sim \mathbi{D}_{h_0}$ and $\{h_1^1, \dots, h_1^{N_1}\} \sim \mathbi{D}_{h_1}$, then unbiased estimation of the squared MMD is
\begin{equation}
    \begin{aligned}
        \mathcal{L}_{\text{MMD}}(\psi_0, \psi_1) =& \frac{1}{N_0 (N_0-1)} \sum_{i\neq i'} k(h_0^i, h_0^{i'}) \\
        &+ \frac{1}{N_1 (N_1-1)} \sum_{j\neq j'} k(h_1^j, h_1^{j'}) \\
        &- \frac{2}{N_0 N_1} \sum_{i} \sum_{j} k(h_0^i, h_1^{j}) \\
    \end{aligned}
    \label{eq:mmd}
\end{equation}

To align state visitation distributions, we optimize the feature extractor using Eq.~\ref{eq:mmd}. In the implementation of the MMD variant, we use a linear combination of RBF kernels with bandwidths $\{0.001, 0.005, 0.01, 0.05, 0.1, 1, 5, 10\}$ since it remains an open problem on choosing the optimal kernels.

\paragraph{Extension for Multi-group Fairness}
We could extend our algorithm for multi-group fairness by learning one feature extractor for each group. The model update step remains the same. To reduce model complexity, we might choose to align the state visitation distributions between the two groups with the largest return disparity. We leave this extension for future study.

\section{Experiments}
In the following, we conduct empirical evaluation on two real-world datasets, showing that our proposed algorithms help to mitigate return disparity while maintaining the performance of policies.\footnote{Our code is publicly available at: \url{https://github.com/JFChi/Return-Parity-MDP}}

\subsection{Experimental Setup}

\paragraph{Datasets} The two real-world datasets we use are benchmark recommender system datasets with user demographic information (\eg, gender and age):
\begin{itemize}[leftmargin=*]
    \item MovieLens-1M\footnote{\url{https://grouplens.org/datasets/movielens/}} is a benchmark dataset consisting of 1 million ratings from more than 6000 users on more than 4000 movies on the MovieLens website. The movie ratings range from 1 to 5 and the users are provided with demographic information such as gender and age. 
    \item Book-Crossing\footnote{\url{http://www2.informatik.uni-freiburg.de/~cziegler/BX/}} is a book rating dataset collected from the Book-Crossing community. It consists of more than 1 million ratings from more than 278k users on about 271k books. The book ratings range from 0 to 10, and the users are provided with demographic information such as age.
\end{itemize}

Our goal of using these datasets is to learn a recommender system that maximizes the expected long-term user satisfaction while keeping the user satisfaction in different demographic groups similar. 

\paragraph{Environment Simulator}

We focus on evaluating our proposed RL algorithms on the two benchmark datasets by setting up an environment simulator~\citep{chen2019large, zhou2020interactive} to mimic online environments. Specifically, we normalize the user ratings in each dataset into the range $[-1, 1]$. Given a user's historical interaction with the recommender system at time step $t$ (state $s_t$), the recommender system recommends an item (action $a_t$), and the user gives a rating to the recommended item (reward $r_t$). Following~\citep{chen2019large}, we give the details of the state representation scheme used in this work in Figure~\ref{fig:state-representation}.

\begin{figure}[!ht]
    \centering
    \includegraphics[width=1.0\linewidth]{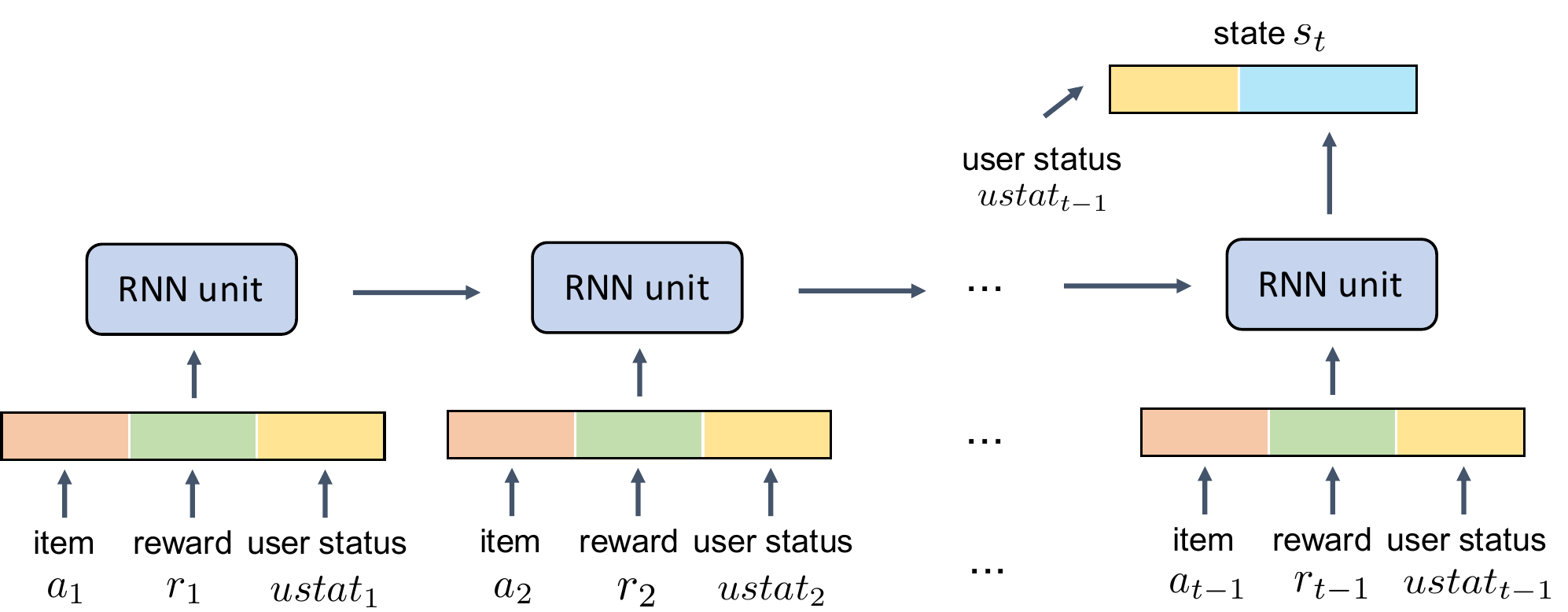}
    \caption{State representation in our RL environments.}
    \label{fig:state-representation}
\end{figure}

\begin{figure*}[!htb]
\centering
\begin{subfigure}[b]{.235\linewidth}
  \centering
  \includegraphics[width=1.1\linewidth]{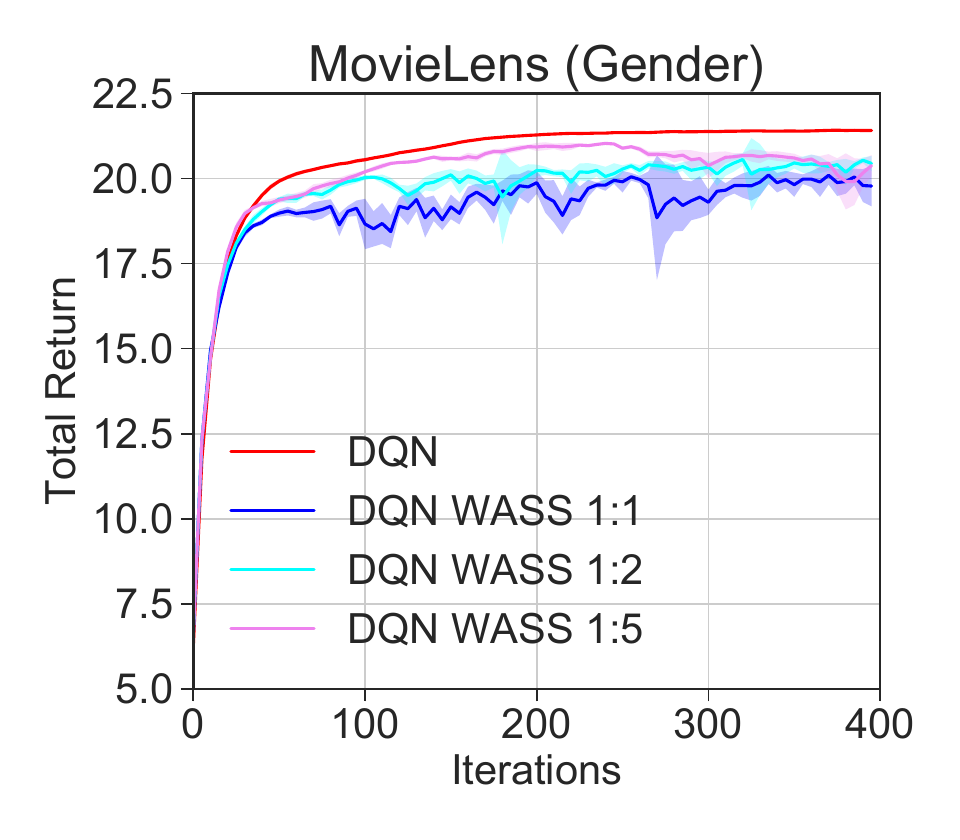}
\end{subfigure}
~
\begin{subfigure}[b]{.235\linewidth}
  \centering
\includegraphics[width=1.1\linewidth]{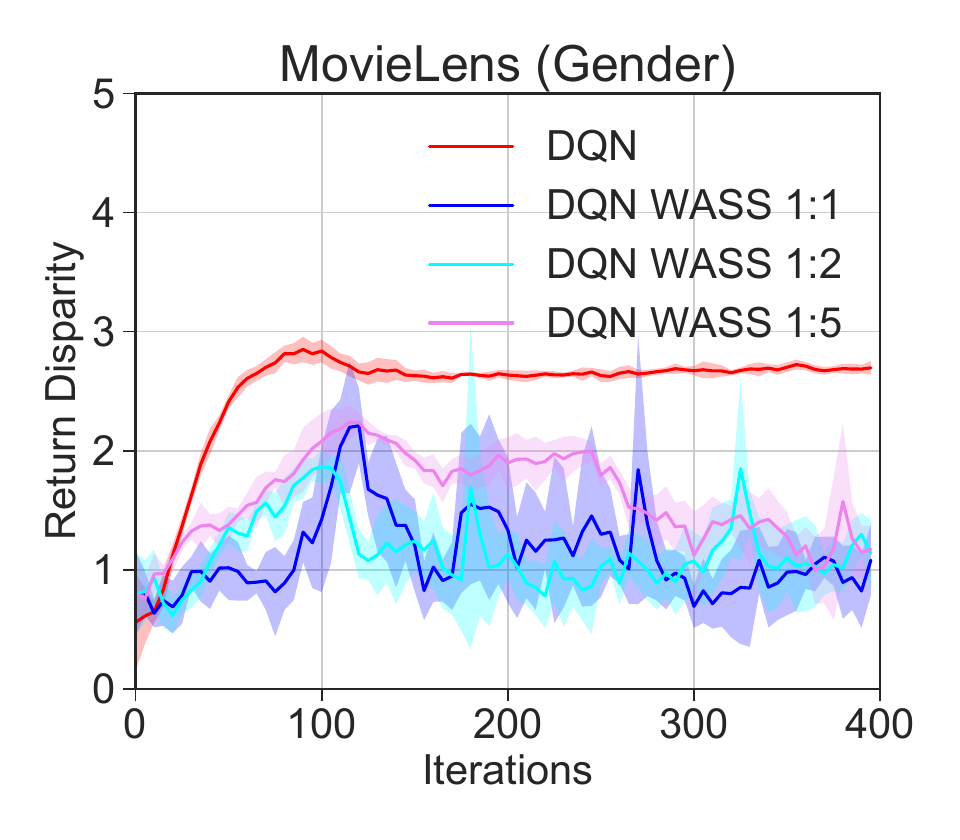}
\end{subfigure}
~
\begin{subfigure}[b]{.235\linewidth}
  \centering
  \includegraphics[width=1.1\linewidth]{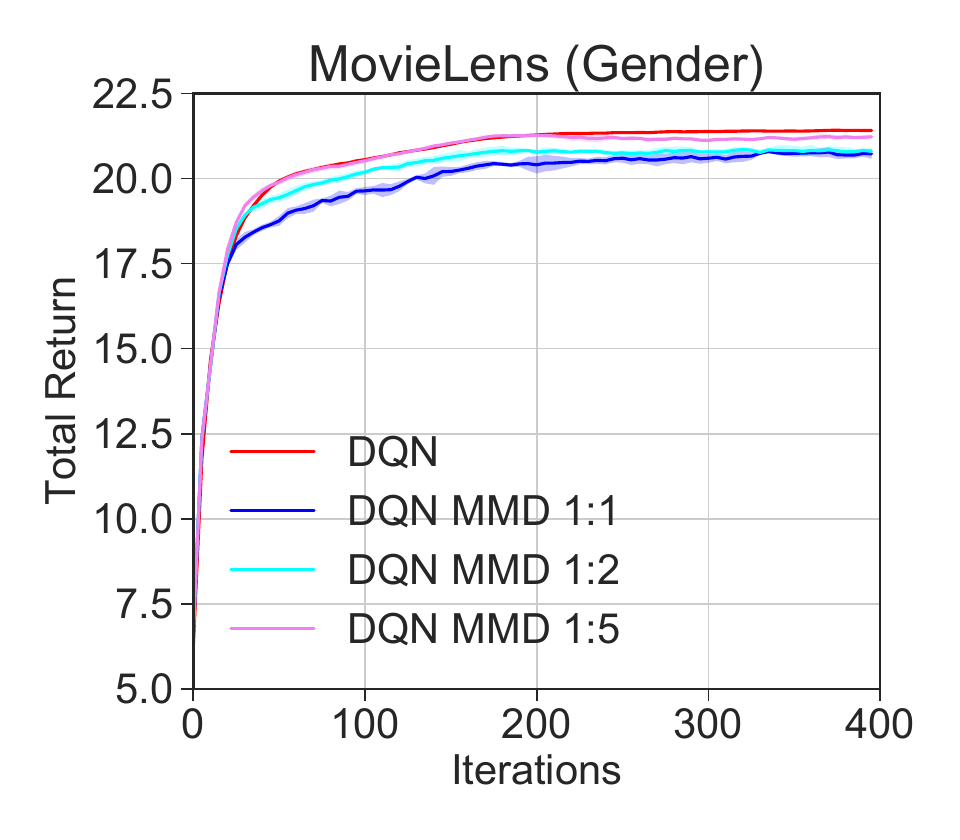}
\end{subfigure}
~
\begin{subfigure}[b]{.235\linewidth}
  \centering
  \includegraphics[width=1.1\linewidth]{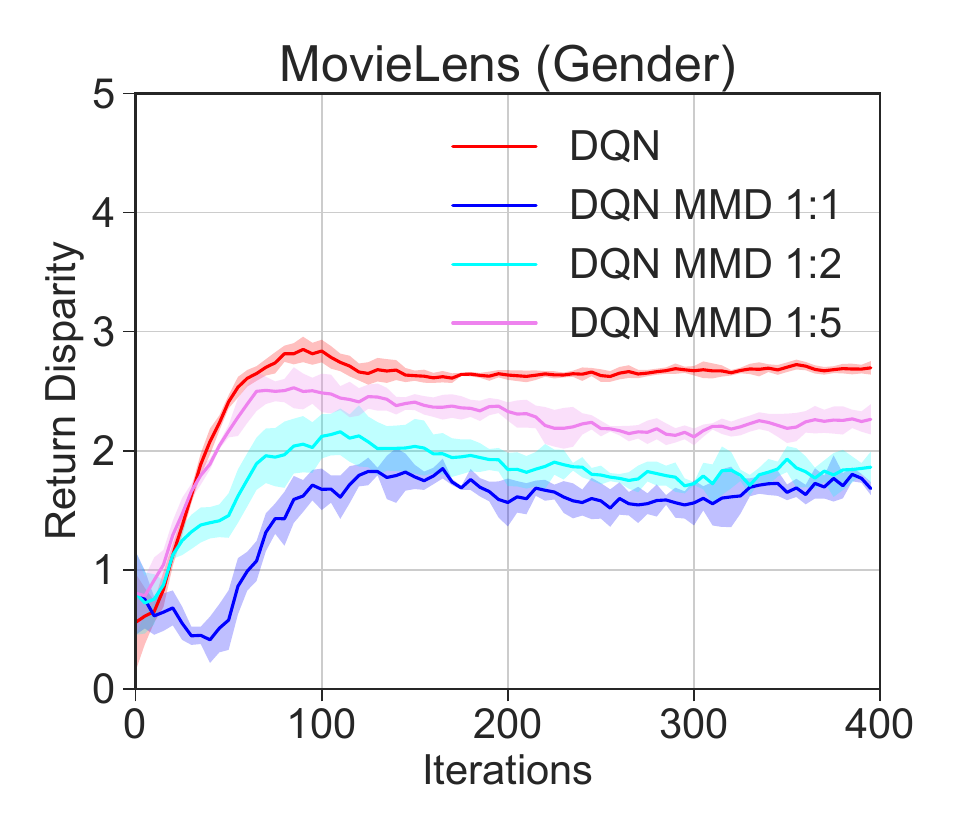}
\end{subfigure}
~
\begin{subfigure}[b]{.235\linewidth}
  \centering
  \includegraphics[width=1.1\linewidth]{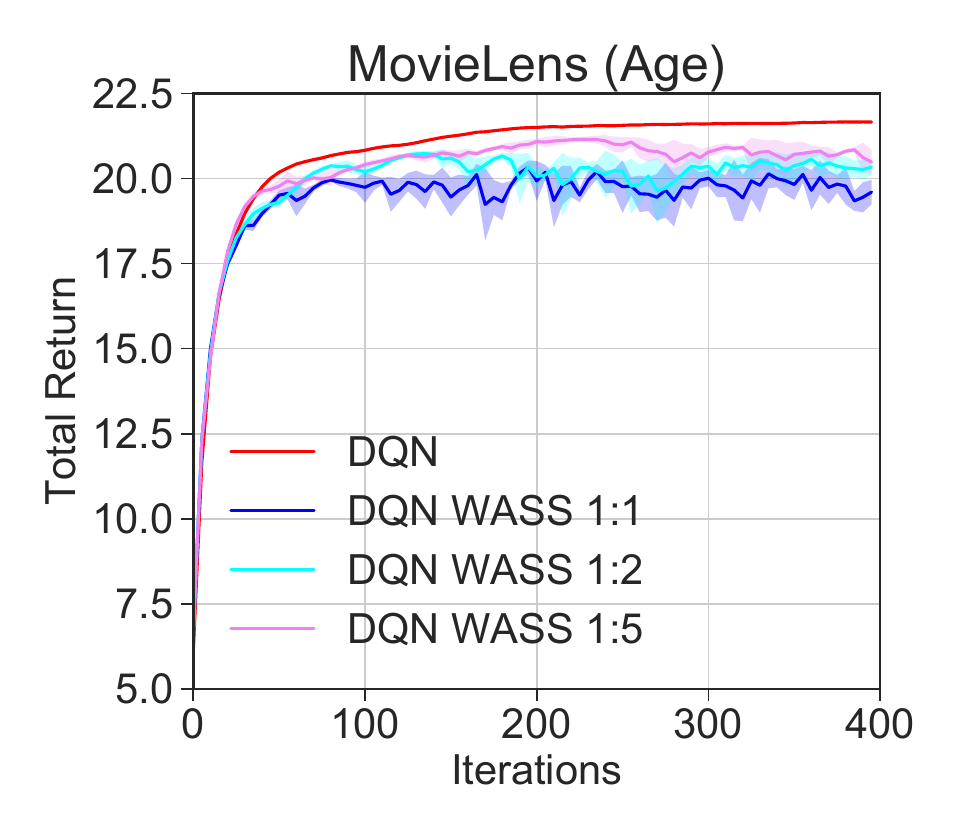}
\end{subfigure}
~
\begin{subfigure}[b]{.235\linewidth}
  \centering
\includegraphics[width=1.1\linewidth]{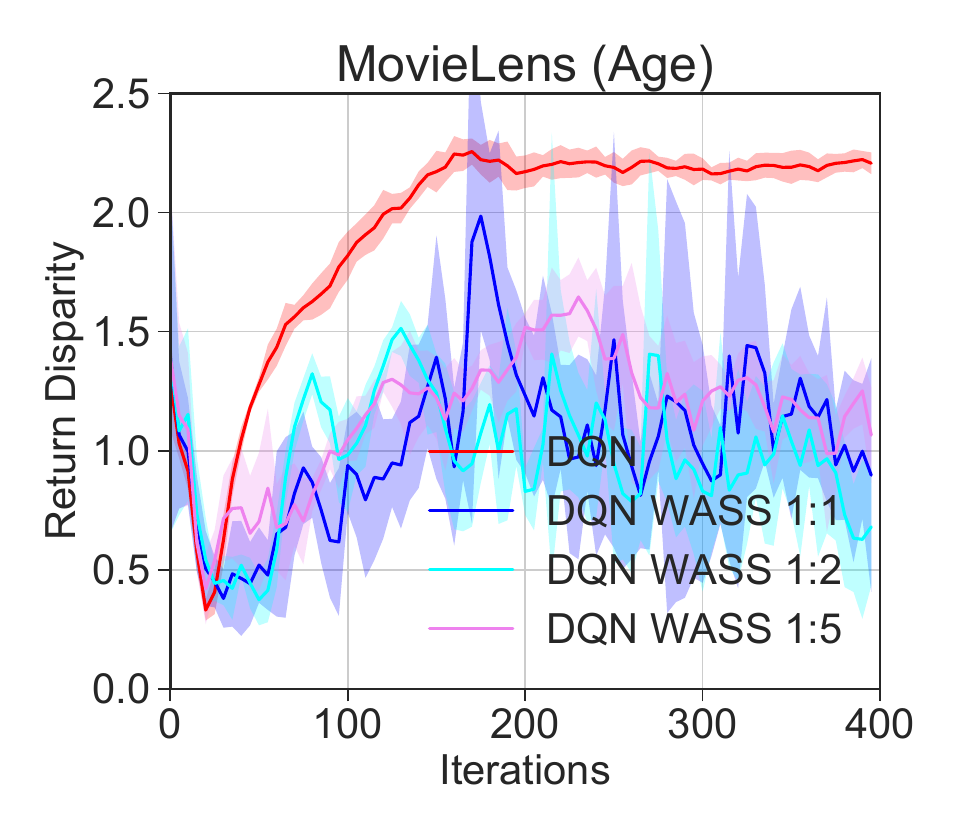}
\end{subfigure}
~
\begin{subfigure}[b]{.235\linewidth}
  \centering
  \includegraphics[width=1.1\linewidth]{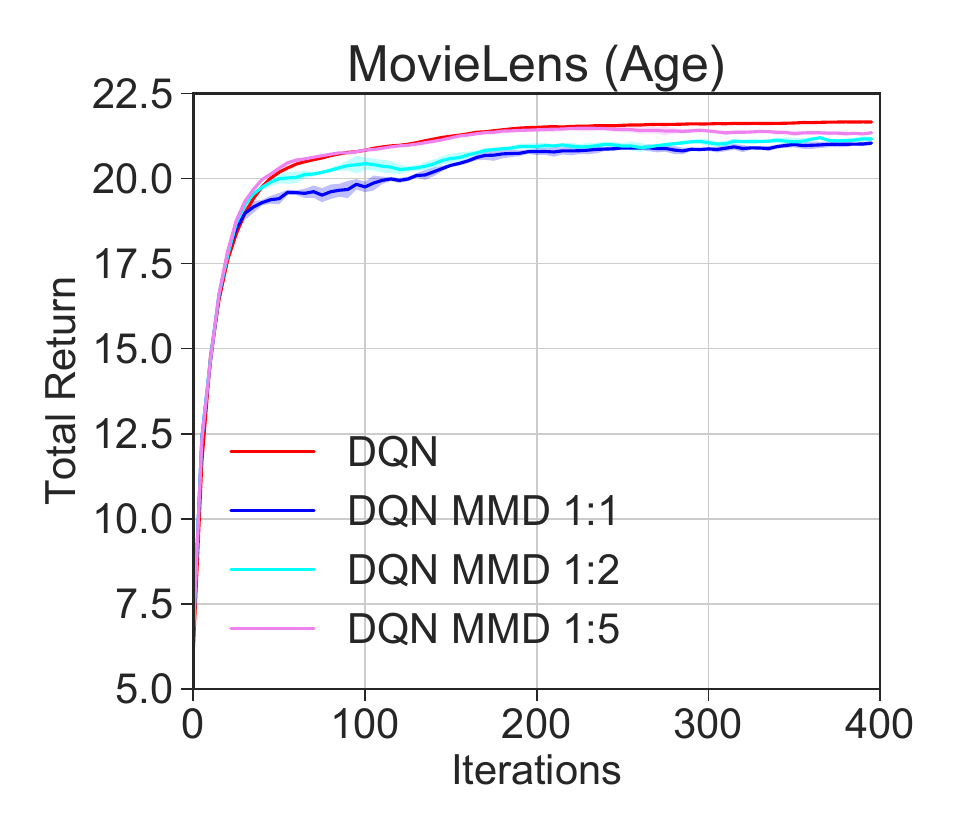}
\end{subfigure}
~
\begin{subfigure}[b]{.235\linewidth}
  \centering
  \includegraphics[width=1.1\linewidth]{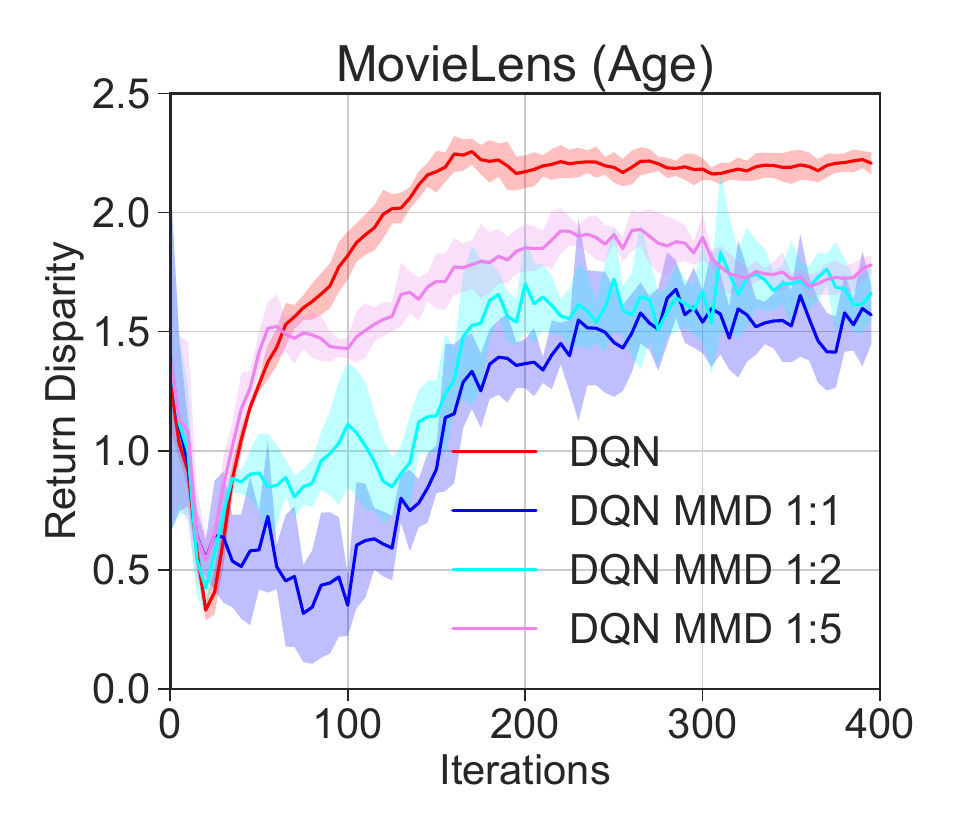}
\end{subfigure}
~
\begin{subfigure}[b]{.235\linewidth}
  \centering
  \includegraphics[width=1.1\linewidth]{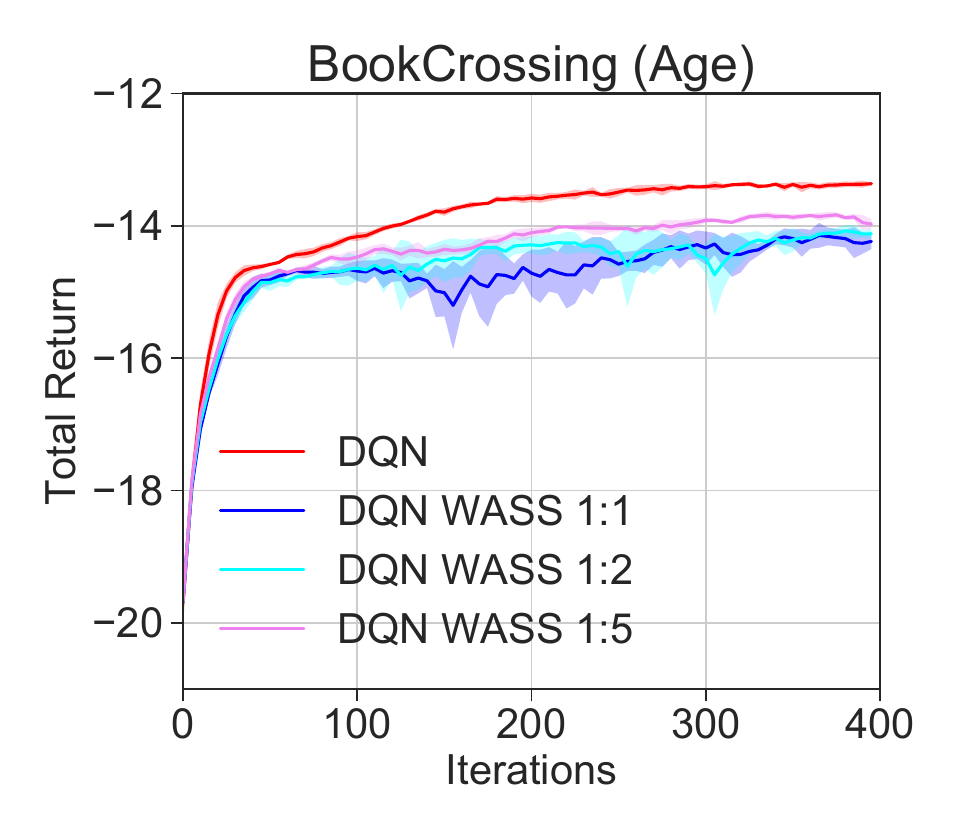}
\end{subfigure}
~
\begin{subfigure}[b]{.235\linewidth}
  \centering
\includegraphics[width=1.1\linewidth]{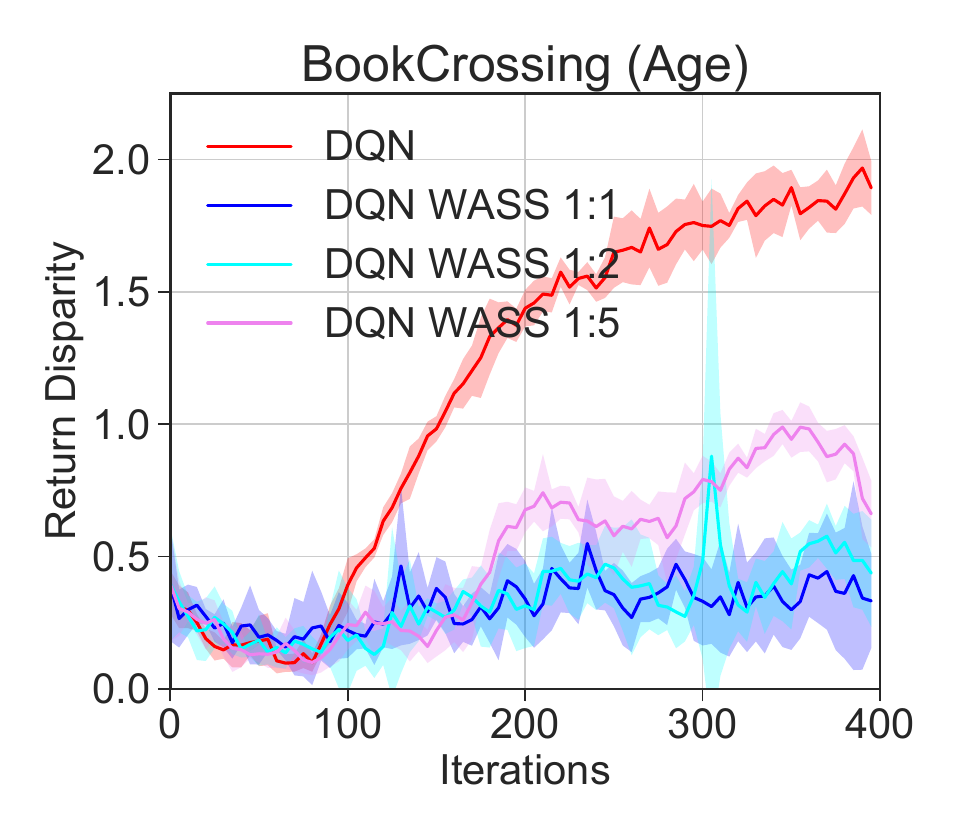}
\end{subfigure}
~
\begin{subfigure}[b]{.235\linewidth}
  \centering
  \includegraphics[width=1.1\linewidth]{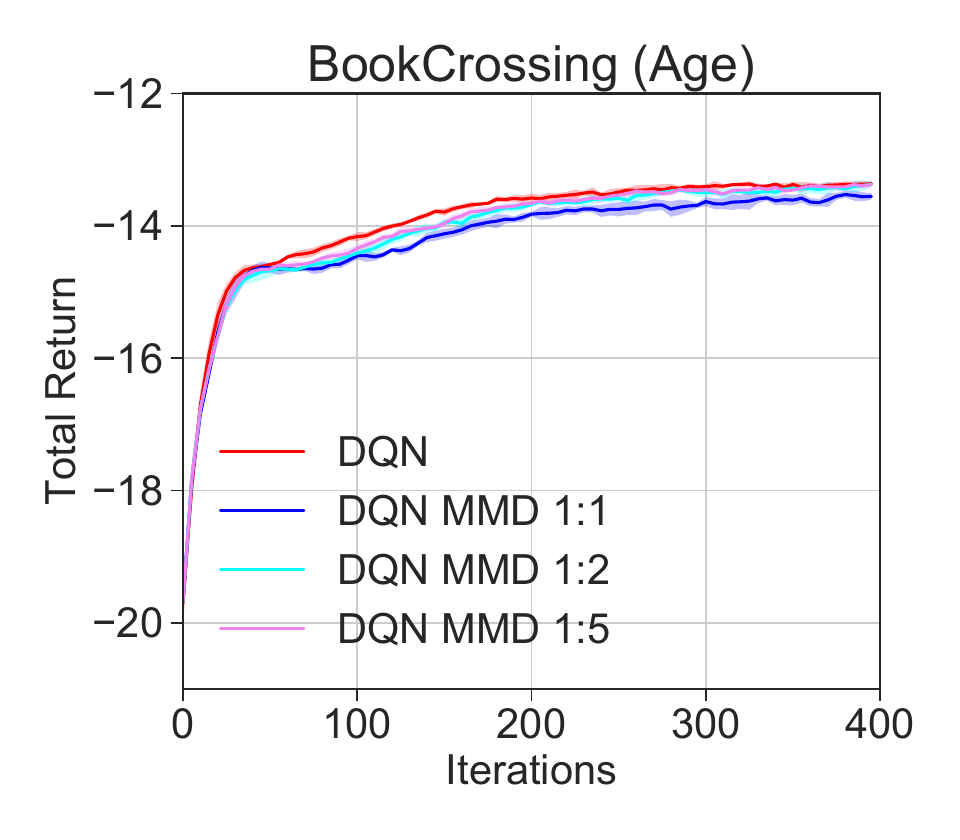}
\end{subfigure}
~
\begin{subfigure}[b]{.235\linewidth}
  \centering
  \includegraphics[width=1.1\linewidth]{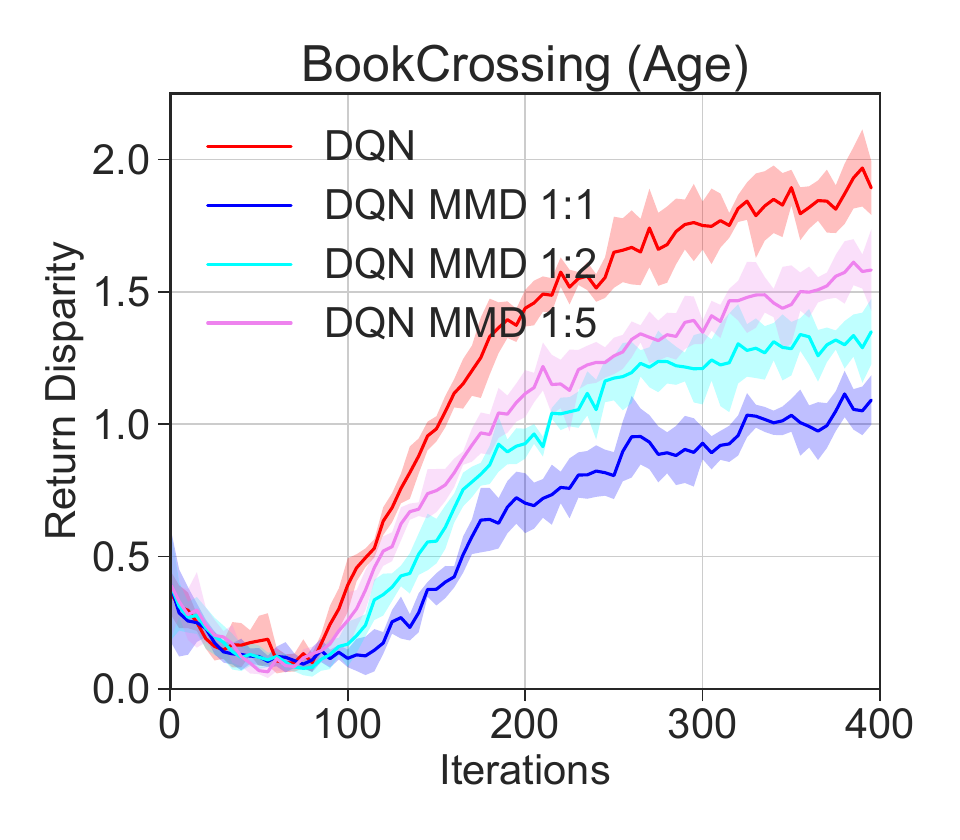}
\end{subfigure}
~
    \caption{Learning curves of DQN, \textsc{DQN-WASS} and \textsc{DQN-MMD} in three different settings.
    The legend \textsc{DQN WASS} (\textsc{DQN MMD}) X:Y indicates the interval of model update versus the interval of state visitation distributional alignment with Wasserstein-1 distance (MMD distance) is X:Y (\ie, smaller numbers means more frequent updates). With the increase of the frequency of state visitation distributional alignment, return disparities are decreasing at the cost of performances of policies in all environments.}
    \label{fig:main-results}
\end{figure*}

As shown in Figure~\ref{fig:state-representation}, the user state $s_t$ is constructed by concatenating the user status at $t-1$ and the $t-1$-th output of a recurrent neural network (RNN). The user status at $t-1$ contains statistics information such as the number of positive/negative rewards before time step $t$. The input of the RNN at each time step is composed of three signals: the recommended item, the reward gained by recommending the corresponding item, and the user status. Note that each item and reward are mapped to an embedding vector and a one-hot vector before inputting to the RNN. We perform matrix factorization~\citep{koren2009matrix} to train an item embedding for each item to recommend.

For each dataset, we randomly split the users into two parts: 80\% of the users are used for training, and the other 20\% are used for testing. Due to the way we perform the train/test split in our datasets, our experiments are cold-start scenarios: the users in the test set have never been seen during training, and there is no interaction between the recommender system and the users at first. To deal with the problem, our recommender system recommends a popular item among the training users to a test user at time step $t_0$ and recommends non-repeated items to
the user interactively according to users' feedback. The episode length is set to be 32 for each user in the two datasets in our experiments.

\paragraph{Methods and Implementation Details}
We implement the DQN algorithm and our proposed algorithms that perform state visitation distributional alignment via Wasserstein-1 distance (\textsc{DQN-WASS}) and maximum mean discrepancy (\textsc{DQN-MMD}). To the best of our knowledge, our work is the first work that studies the return disparity problem in MDPs. We also adapt the state-of-the-art reduction approach, constrained policy optimization (CPO)~\citep{achiam2017constrained} to our problem setting and find its training processes cannot converge, so we do not include the results here. We suspect the failure of CPO is because their setting is different from ours: they assume the constraint of the policy is a deterministic function determined by states and actions, while in our setting, the reward gap in each environment step is dynamically changing, making it hard to estimate.

We take gender and age in MovieLens dataset and age in Book-Crossing dataset as the binary demographic groups (\eg, male/female, young/old). In each environment, we vary the model update frequency and the state visitation distributional alignment update frequency in \textsc{DQN-Wass} and \textsc{DQN-MMD} and report the corresponding overall return and the return disparity between groups. We average the results over five different random seeds and visualize the performance curves in each setting.
The detailed data pre-processing pipelines and hyper-parameter configurations in our experiments are presented in Appendix~\ref{app-sec:exp-detail}. 

\subsection{Results and Analysis}

The performance curves of \textsc{DQN}, \textsc{DQN-WASS} and \textsc{DQN-MMD} are shown in Figure~\ref{fig:main-results}. 
We can see that with the increase of the frequency of state visitation distributional alignment, return disparities are decreasing at the cost of performances of policies in all environments, which demonstrates the trade-offs between return maximization and return parity.
Our methods can flexibly tune the trade-off between return maximization and return parity by controlling the relative update frequency ratio between the model update step and state visitation distributional alignment update step. 
Compared to \textsc{DQN-WASS}, \textsc{DQN-MMD} leads to more stable training processes with slightly higher overall returns and return disparities on average. 

Next, we provide more insights into how our algorithms work. Since the MMD distance can be estimated analytically using Eq.~\ref{eq:mmd}, we visualize the learning curves of estimated MMD distances between the (induced) state visitation distributions in different MMD update settings in MovieLens (Gender) in Figure~\ref{fig:mmd-curves}. We can see that the more frequent the MMD update is, the smaller the MMD distance. 
We also perform principal component analysis (PCA) on sampled state representations of different groups for \textsc{DQN} and \textsc{DQN-MMD} in MovieLens (Gender). The visualization results are presented in Figure~\ref{fig:mmd-pca-visitation}. We can see from Figure~\ref{fig:mmd-pca-visitation} that \textsc{DQN-MMD} helps to align the state visitation distributions of different groups, which is consistent with our theoretical findings in Theorem~\ref{thm:main}.

\begin{figure}[!t]
    \centering
    \includegraphics[width=0.75\linewidth]{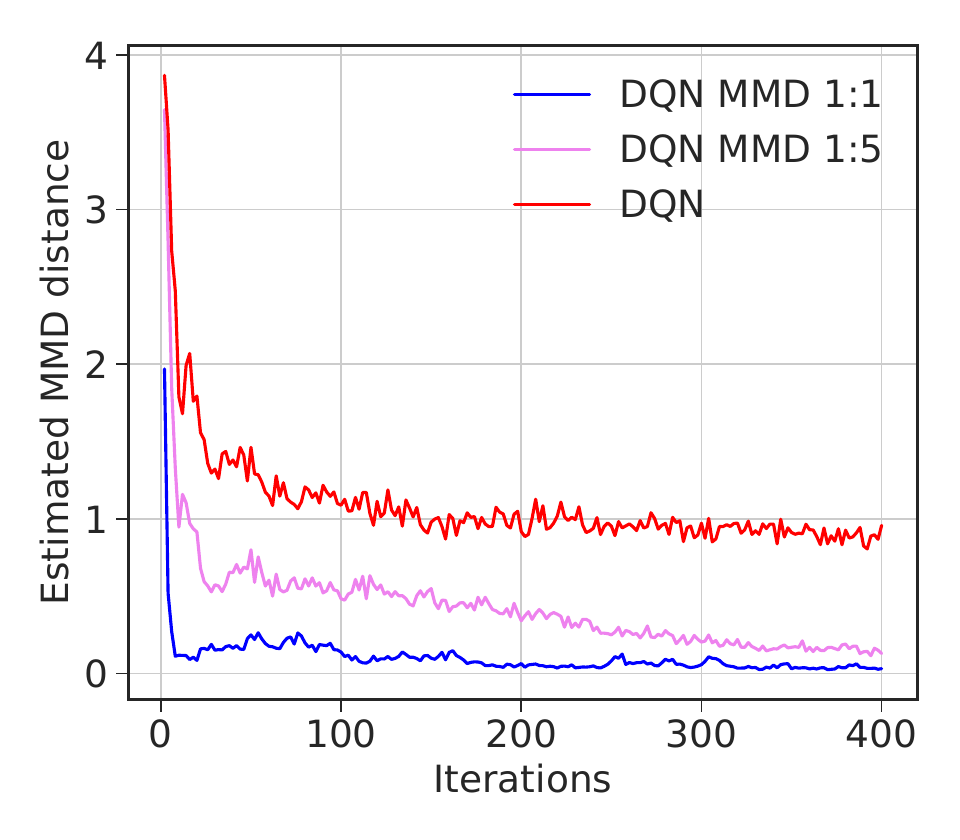}
    \vspace{-1em}
    \caption{Training visualization of estimated MMD distance in MovieLens (Gender). With the increased frequency of MMD update, the MMD distance becomes smaller.}
    \vspace{-1em}
    \label{fig:mmd-curves}
\end{figure}

\begin{figure}[!t]
\begin{subfigure}[b]{0.235\textwidth}
    \includegraphics[width=\textwidth]{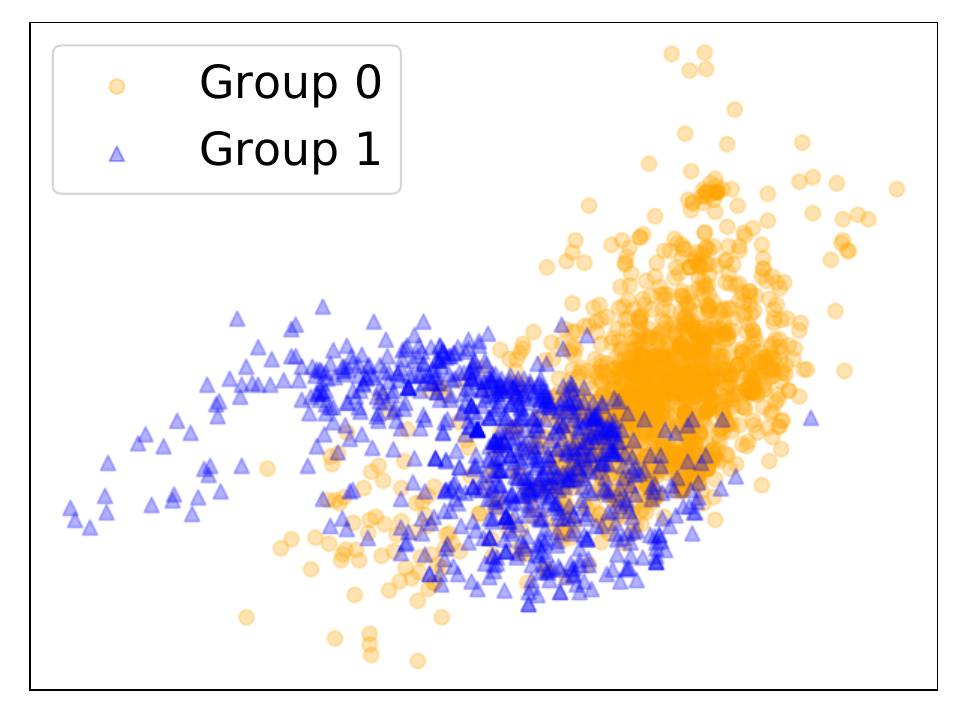}
    \caption{\textsc{DQN}}
  \end{subfigure}
  \begin{subfigure}[b]{0.235\textwidth}
\includegraphics[width=\textwidth]{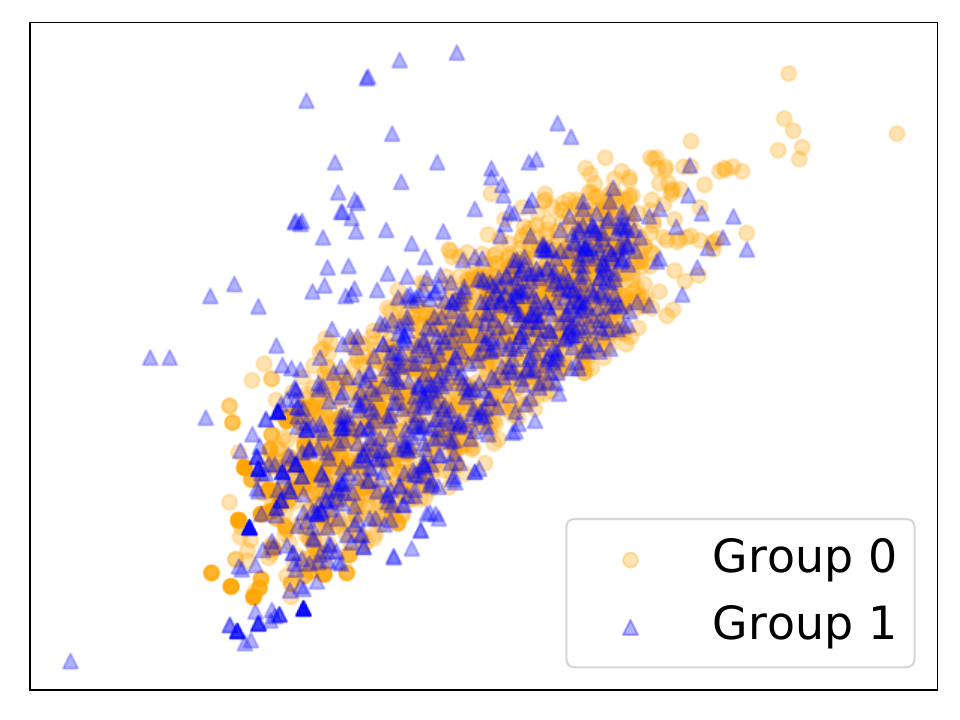}
    \caption{\textsc{DQN-MMD}}
  \end{subfigure}
\caption{PCA visualization of sampled state visitation representations of different groups for \textsc{DQN} and \textsc{DQN-MMD} in MovieLens (Gender). Our method helps to align the state visitation distributions of different groups.}
  \label{fig:mmd-pca-visitation}
\end{figure}

\section{Related Work}

\paragraph{Fairness in MDPs}
\citet{jabbari2017fairness} propose an individual fairness  notion which requires an algorithm to select an action if the long-term (discounted) reward of choosing that action is higher than the others in MDPs;
\citet{wen2019fairness} extend static fairness notions to Markov decision processes and propose model-based and model-free algorithms to maximize expected return while satisfying the fairness constraints. However, their proposed algorithms are based on linear programming or evolutionary algorithm, which cannot be scaled up to solve complex tasks compared to deep reinforcement learning approaches;
\citet{siddique2020learning} consider the fairness of multi-objective MDPs and propose to learn a policy with objective function satisfy the generalized Gini social welfare function~\citep{weymark1981generalized}; \citet{ge2021towards} consider the fairness of item exposure problem in recommendation systems and extend CPO algorithm~\citep{achiam2017constrained} to satisfy the fairness constraint; \citet{thomas2019preventing} propose an offline RL algorithm called quasi-Sheldonian reinforcement learning algorithm to determine whether given sets of policy distributions satisfy a set of return constraints with guarantees; 
\citet{zhang2020fair} analyze how static fairness constraints affect the dynamics of group qualification rates. 
In contrast to their work, we propose the notion of return parity to quantify the dynamics of performance disparity in general changing environments: we theoretically analyze the notion of return parity by providing sufficient conditions where return parity can be satisfied and propose a decomposition theorem for return disparity. Based on our decomposition theorem, we propose algorithms to mitigate return disparity and empirically show the effectiveness of our algorithms.

\paragraph{Fairness under Other Temporal Models}
A series of works focus on study fairness in the setting of online learning~\citep{blum2018preserving, bechavod2019equal, gupta2019individual}, multi-armed bandit~\citep{joseph2016fairness, patil2020achieving, chen2020fair}, and one-step feedback model\citep{liu2018delayed, hu2018short, kannan2019downstream}. 
In particular, \citet{pmlr-v80-hashimoto18a} show that empirical risk minimization amplifies representation disparity over time with a low group retention rate for the underrepresented group. They further propose distributionally robust optimization to minimize the worst-case risk overall group distributions. 
\citet{heidari2018preventing} propose a time-dependent individual fairness notion that requires similar individuals should receive similar outcomes during the same time epoch.
\citet{ensign2018runaway} model predictive policing problem using P\'{o}lya urn model and show that all police officers will be allocated to one location if more officers are constantly assigned to the locations with higher predicted crime rates.
\citet{heidari2019long} propose effort-based fairness, which measures unfairness as the disparity in the effort made by individuals from each group to get a target outcome.
\citet{zhang2019group} quantify the condition of the exacerbation of relative group ratio under the fairness constraints such as demographic parity and equalized odds.

\section{Limitations and Future Work}

One limitation of our approach is that experimental results show trade-offs between return parity and return maximization, which is possible if state (feature) visitation distributions induced by the optimal group policies for the MDPs are largely different. As a future direction, we plan to propose more stable and efficient algorithms to achieve Pareto optimality for return maximization and mitigation of return disparity beyond binary demographic groups.

It is also worth noting that in the special case when the length of time horizon is 1, then our notion of return parity corresponds to accuracy parity in classification settings (\ie, return maximization is reduced to accuracy maximization). However, it remains unclear whether our proposed fairness notion is compatible with other group fairness notions, such as demographic parity and equalized odds. We also leave this analysis as future work as it is a fundamental question that warrants an independent study.

\section{Conclusions}
\label{sec:conclusion}
In this paper, we investigate the problem of return parity in MDPs both theoretically and empirically. In particular, we prove a decomposition theorem for return disparity which decomposes the return disparity of any two MDPs into the distance between group-wise reward functions, the discrepancy of group policies, and the discrepancy between state visitation distributions. We then provide algorithmic interventions to mitigate return disparity via state visitation distributional alignment with IPMs. To corroborate our theoretical results, we conduct experiments on two real-world benchmark datasets. Experimental results suggest that our proposed algorithms help to mitigate return disparity while maintaining the performance of policies. We believe our work takes an important step towards better understanding the dynamics of performance disparity in changing environments. 

\subsubsection*{Acknowledgements}

We thank the anonymous reviewers for their insightful comments.
Jian Shen, Xinyi Dai, and Weinan Zhang acknowledge support from ``New Generation of AI 2030'' Major Project (2018AAA0100900), Shanghai Municipal Science and National Natural Science Foundation of China (62076161).
Jianfeng Chi and Yuan Tian acknowledge support from NSF 1920462, 1943100, 2002985, and a Google research scholar award. Han Zhao would like to thank support from a Facebook research award.

\bibliography{reference}
\bibliographystyle{plainnat}

\newpage
\onecolumn
\appendix
\section*{Appendix}

\section{Missing Proofs}
\label{app-sec:proof}

\subsection{Proof of Proposition~\ref{prop:imp}}

\ImpossiblityResults*

\begin{proof}
Consider two MDPs share two states $s_1$ and $s_2$. Let $r(s_1, a)= c (1-\gamma) > 0$ and $r(s_2, a) = 0,~\forall a\in\aaspace$, $T(s_2\mid s_1, a)=T(s_1\mid s_2, a) = 0$ and $T(s_1\mid s_1, a)=T(s_2\mid s_2, a) = 1,~\forall a\in\aaspace$. Given $\mu_0 = [1, 0]^T$ and $\mu_1 = [0, 1]^T$, then the expected return for group $0$ is $c$, while the expected return for group $1$ is $0$. In this case, the return gap $\retgap = c \geq 0.$
\label{prop:impossibility}
\end{proof}

\subsection{Proof of Proposition~\ref{prop:suff_v2}}

\PossiblityResults*

\begin{proof}
Let $\mu_0$ and $\mu_1$ be initial distributions of group $0$ and $1$, respectively, and denote $V^{\pi_0} = [v^{\pi_0}(s_1), \dots, v^{\pi_0}(s_m)]^T$ and $V^{\pi_1} = [v^{\pi_1}(s_1), \dots, v^{\pi_1}(s_m)]^T$. By definition, in order to satisfy return parity, the following equation should be satisfied:  
\begin{equation}
    \nonumber
    \mu_0^T V^{\pi_0} = \mu_0^T V^{\pi_1}.
\end{equation}

Consider solving the optimal value function via linear programming while satisfying the constraint of return parity, then the whole optimization problem becomes 
\begin{equation}
    \nonumber
    \begin{aligned}
    \min_{V^{\pi_0}, V^{\pi_1}}&\quad \lambda \mu_0^T V^{\pi_0} + (1-\lambda) \mu_1^T V^{\pi_1}\\
    \text{s.t.}&\quad v^{\pi_0}(s) \geq r_0(s, a) + \gamma \sum_{s'} T_0(s'\mid s, a) v^{\pi_0}(s'), \quad\forall a \in \aaspace, s \in \sspace\\
    &\quad v^{\pi_1}(s) \geq r_1(s, a) + \gamma \sum_{s'} T_1(s'\mid s, a) v^{\pi_1}(s'), \quad\forall a \in \aaspace, s \in \sspace\\
    &\quad | \mu_0^T V^{\pi_0} - \mu_1^T V^{\pi_1} | \leq \epsilon\\
    \end{aligned}
\end{equation} 

Next, we convert the constraints of the LP in the standard form:
\begin{equation}
    \nonumber
    \begin{aligned}
    ( \gamma T_0(s\mid s, a) - 1 ) v^{\pi_0}(s) + \gamma \sum_{s'\neq s} T_0(s'\mid s, a) v^{\pi_0}(s') &\leq - r_0(s, a) \quad\forall a \in \aaspace, s \in \sspace\\
    ( \gamma T_1(s\mid s, a) - 1 ) v^{\pi_1}(s) + \gamma \sum_{s'\neq s} T_1(s'\mid s, a) v^{\pi_1}(s') &\leq - r_1(s, a) \quad\forall a \in \aaspace, s \in \sspace\\
    \mu_0^T V^{\pi_0} - \mu_1^T V^{\pi_1} &\leq \epsilon\\
    \mu_1^T V^{\pi_1} - \mu_0^T V^{\pi_0} &\leq \epsilon.
    \end{aligned}
\end{equation}

The variant of Farkas' Lemma states that either the system $\mathbf {Ax} \leq \mathbf {b}$ has a solution with $\mathbf {x} \geq 0$, or the system $\mathbf {A} ^{\mathsf {T}}\mathbf {y} \geq 0$ has a solution with $\mathbf {b} ^{\mathsf {T}}\mathbf {y} <0$ and $\mathbf {y} \geq 0$. In other words, if we want the above system has a non-negative solution,
we can easily show that $\nexists~\hat{\rho}_0 (s, a), \hat{\rho}_1 (s, a), b_0, b_1\geq 0$, $\forall~i\in[m]$ such that 
\begin{equation}
    \nonumber
    \begin{aligned}
     \sum_s \sum_a T_0(s_i \mid s, a) \hat{\rho}_0 (s, a) - \sum_a \hat{\rho}_0 (s_i, a) + (b_0 - b_1) (\mu_0)_i &\geq 0\\
     \sum_s \sum_a T_1(s_i \mid s, a) \hat{\rho}_1 (s, a) - \sum_a \hat{\rho}_1 (s_i, a) + (b_1 - b_0) (\mu_1)_i &\geq 0\\
     \sum_s \sum_a \hat{\rho}_0 (s, a) r_0(s, a) + \hat{\rho}_1 (s, a) r_1(s, a) - (b_0 + b_1) \epsilon &> 0\\
    \end{aligned}
\end{equation}

By the law of contraposition, it is equivalent to its contrapositive: $\forall~\hat{\rho}_0 (s, a), \hat{\rho}_1 (s, a), b_0, b_1\geq 0$, $\exists~i\in[m]$ such that
\begin{equation}
    \nonumber
    \begin{aligned}
     \sum_s \sum_a T_0(s_i \mid s, a) \hat{\rho}_0 (s, a) - \sum_a \hat{\rho}_0 (s_i, a) + (b_0 - b_1) (\mu_0)_i &< 0\\
     \sum_s \sum_a T_1(s_i \mid s, a) \hat{\rho}_1 (s, a) - \sum_a \hat{\rho}_1 (s_i, a) + (b_1 - b_0) (\mu_1)_i &< 0\\
     \sum_s \sum_a \hat{\rho}_0 (s, a) r_0(s, a) + \hat{\rho}_1 (s, a) r_1(s, a) - (b_0 + b_1) \epsilon &\leq 0\\
    \end{aligned}
\end{equation}

Reorganizing the above equations then completes the proof.
\end{proof}

\subsection{Proof of Theorem~\ref{thm:main}}

\rewardDisparityUpperBoundState*

\begin{proof}
By definition, the return disparity is
\begin{equation}
    \nonumber
    \begin{aligned}
    \retgap=&~ \big|\Exp_{s\in \mu_0}[v^{\pi}(s)] - \Exp_{s\in \mu_1}[v^{\pi}(s)] \big| \\
    =&~ \frac{1}{1-\gamma} \bigg| \sum_{s\in \sspace}\sum_{a\in \aaspace} r_0(s, a) \rho^{\pi_0}(s, a) - \sum_{s\in \sspace}\sum_{a\in \aaspace} r_1(s, a) \rho^{\pi_1}(s, a) \bigg| \\
    =&~ \frac{1}{1-\gamma} \bigg| \sum_{s\in \sspace}\sum_{a\in \aaspace} r_0(s, a) \rho^{\pi_0}(s, a) - r_1(s, a) \rho^{\pi_1}(s, a) \bigg|\\
    =&~ \frac{1}{1-\gamma} \bigg| \sum_{s\in \sspace}\sum_{a\in \aaspace} r_0(s, a)\rho^{\pi_0}(s, a) - r_1(s, a)\rho^{\pi_0}(s, a) + r_1(s, a)\rho^{\pi_0}(s, a) - r_1(s, a) \rho^{\pi_1}(s, a) \bigg|\\
    \leq&~ \frac{1}{1-\gamma} \bigg| \sum_{s\in \sspace}\sum_{a\in \aaspace} r_0(s, a)\rho^{\pi_0}(s, a) - r_1(s, a)\rho^{\pi_0}(s, a)\bigg| + \frac{1}{1-\gamma} \bigg| \sum_{s\in \sspace}\sum_{a\in \aaspace} r_1(s, a)\rho^{\pi_0}(s, a) - r_1(s, a) \rho^{\pi_1}(s, a) \bigg|,\\
    \end{aligned}
\end{equation}
where the first term is upper bounded by 
\begin{equation}
    \nonumber
    \begin{aligned}
    \bigg| \sum_{s\in \sspace}\sum_{a\in \aaspace} r_0(s, a)\rho^{\pi_0}(s, a) - r_1(s, a)\rho^{\pi_0}(s, a)\bigg| =&~\bigg| \sum_{s\in \sspace}\sum_{a\in \aaspace} \big(r_0(s, a)- r_1(s, a) \big) \rho^{\pi_0}(s, a)\bigg| \\
    \leq& \max_{s, a} | r_0(s, a) - r_1(s, a) | \sum_{s\in \sspace}\sum_{a\in \aaspace} \rho^{\pi_0}(s, a) = \|r_0 - r_1\|_\infty,\\
    \end{aligned}
\end{equation}
and the second term is upper bounded by
\begin{equation}
    \begin{aligned}
     &~\bigg| \sum_{s\in \sspace}\sum_{a\in \aaspace} r_1(s, a)\rho^{\pi_0}(s, a) - r_1(s, a) \rho^{\pi_1}(s, a) \bigg|\\
     =&~\bigg| \sum_{s\in \sspace}\sum_{a\in \aaspace} r_1(s, a)\mu^{\pi_0}(s) \pi_0(a \mid s) - r_1(s, a) \mu^{\pi_1}(s) \pi_1(a \mid s) \bigg| \\
     =&~\bigg| \sum_{s\in \sspace}\sum_{a\in \aaspace} r_1(s, a)\mu^{\pi_0}(s) \pi_0(a \mid s) - r_1(s, a)\mu^{\pi_0}(s) \pi_1(a \mid s) + r_1(s, a)\mu^{\pi_0}(s) \pi_1(a \mid s) - r_1(s, a) \mu^{\pi_1}(s) \pi_1(a \mid s) \bigg|  \\
     \leq&~\bigg| \sum_{s\in \sspace}\sum_{a\in \aaspace} r_1(s, a)\mu^{\pi_0}(s) \pi_0(a \mid s) - r_1(s, a)\mu^{\pi_0}(s) \pi_1(a \mid s) \bigg| \\
     &~~ + \bigg| \sum_{s\in \sspace}\sum_{a\in \aaspace} r_1(s, a)\mu^{\pi_0}(s) \pi_1(a \mid s) - r_1(s, a) \mu^{\pi_1}(s) \pi_1(a \mid s) \bigg|. \\
    \end{aligned}
    \label{eq:proof-main}
\end{equation}

The first term in the upper bound of~\ref{eq:proof-main} is
\begin{equation}
    \nonumber
    \begin{aligned}
    \bigg| \sum_{s\in \sspace}\sum_{a\in \aaspace} r_1(s, a)\mu^{\pi_0}(s) \pi_0(a \mid s) - r_1(s, a)\mu^{\pi_0}(s) \pi_1(a \mid s) \bigg|\leq&~R~  \sum_{s\in \sspace}\sum_{a\in \aaspace}  \mu^{\pi_0}(s) \bigg| \pi_0(a \mid s) - \pi_1(a \mid s) \bigg|\\
    \leq&~ R~ \Exp_{s\sim\mu^{\pi_0}}\bigg[ \sum_{a\in\aaspace} \bigg| \pi_0(a \mid s) - \pi_1(a \mid s) \bigg| \bigg] \\
    =&~ R~ \Exp_{s\sim\mu^{\pi_0}}\bigg[ \sum_{a\in\aaspace} \|\pi_0(\cdot\mid s) - \pi_1(\cdot\mid s)\|_1 \bigg] , \\
    \end{aligned}
\end{equation}
and the second term in the upper bound of~\ref{eq:proof-main} is
\begin{equation}
    \nonumber
    \begin{aligned}
     &~\bigg| \sum_{s\in \sspace}\sum_{a\in \aaspace} r_1(s, a)\mu^{\pi_0}(s) \pi_1(a \mid s) - r_1(s, a) \mu^{\pi_1}(s) \pi_1(a \mid s) \bigg| \\
     =&~ \bigg| \sum_{s\in \sspace} \mu^{\pi_0}(s) \sum_{a\in \aaspace} r_1(s, a) \pi_1(a \mid s) - \sum_{s\in \sspace} \mu^{\pi_1}(s) \sum_{a\in \aaspace} r_1(s, a) \pi_1(a \mid s) \bigg| \\
     =&~ \bigg| \sum_{s\in \sspace} \mu^{\pi_0}(s)~r_1(s)  - \sum_{s\in \sspace} \mu^{\pi_1}(s)~r_1(s) \bigg| \\
     \leq&~ \sup_{r_1(s) \in \mathcal{F}} \bigg| \sum_{s\in \sspace} \mu^{\pi_0}(s)~r_1(s)  - \sum_{s\in \sspace} \mu^{\pi_1}(s)~r_1(s) \bigg| \\
     =&~ d_{\mathcal{F}} (\mu^{\pi_0}(s), \mu^{\pi_1}(s)). \\
    \end{aligned}
\end{equation}
By symmetry of \ref{eq:proof-main}, we also have 
\begin{equation}
    \nonumber
    \begin{aligned}
     &~\bigg| \sum_{s\in \sspace}\sum_{a\in \aaspace} r_1(s, a)\rho^{\pi_0}(s, a) - r_1(s, a) \rho^{\pi_1}(s, a) \bigg|\\
     \leq&~ R~ \Exp_{s\sim\mu^{\pi_1}}\bigg[ \sum_{a\in\aaspace} \|\pi_0(\cdot\mid s) - \pi_1(\cdot\mid s)\|_1 \bigg] + d_{\mathcal{F}} (\mu^{\pi_0}(s), \mu^{\pi_1}(s)) \\
    \end{aligned}
\end{equation}
Combining the above results then completes the proof.
\end{proof}

\newpage
\section{Experimental Details}
\label{app-sec:exp-detail}

\paragraph{MovieLen-1M}
The original rating matrix of the dataset is a sparse matrix. In order to learn a good item embedding for each user, we first filter out the items which get less than 64 ratings (the dataset has ensured that each user has at least 20 ratings). We take gender (\eg, male/female) and age (\eg, $\geq45$~/~$<45$ years old) as the binary demographic groups. Similar to~\citep{zhou2020interactive}, we also perform the matrix completion~\citep{koren2009matrix} on the original rating matrix to avoid the sparsity of rating signals in the original data.
To mimic the real-world scenarios where the number of users could be skewed across different demographic groups, we downsample one group’s data by a factor of 10. Similar to~\citep{chen2019large, liu2020end}, we fix the item embedding when updating our models and pre-train the RNNs.\footnote{\url{https://github.com/chenhaokun/TPGR}} The dimension of the item embedding vector, reward vector and user status (\ie, the inputs of RNNs) are 50, 20, and 9, respectively. Thus, the state dimension is 88.
We give the details of training hyperparameters in Table~\ref{table:hyper}.

\paragraph{Book-Crossing}
Similar to MovieLen-1M, we first filter out the items which get less than 32 ratings and users who have less than 16 ratings. We take age (\eg, $\geq35$~/~$<35$ years old) as the binary demographic groups. The rest of data pre-processing pipelines are similar to those in the MovieLen dataset. The dimension of the item embedding vector, reward vector and user status (\ie, the inputs of RNNs) are 60, 20 and 9, respectively. Thus, the state dimension is 98.
We give the details of training hyperparameters in Table~\ref{table:hyper}.

\begin{table}[htbp]
    \centering
    \begin{tabular}{lccc}
    \toprule
     & MovieLens (Gender) & MovieLens (Age) & Book-Crossing (Age) \\
     \midrule
     Q-network &  \multicolumn{3}{c}{MLP with hidden size [128]} \\
     Soft-update frequency $\tau$ &  \multicolumn{3}{c}{0.99}  \\
     Iteration & \multicolumn{3}{c}{400} \\
     Learning Rate &  \multicolumn{3}{c}{1e-3}  \\
     Q-network Weight Decay &  \multicolumn{3}{c}{1e-6}  \\
     $\varepsilon$-greedy Policy &  \multicolumn{3}{c}{Linear decay from max $\varepsilon=1.0$ to min $\varepsilon=0.1$ with decay steps 160} \\
     Sample Batch Size & \multicolumn{3}{c}{1000}  \\
     Update Batch Size &  \multicolumn{3}{c}{10000}  \\
     Update per Iteration &  \multicolumn{3}{c}{10}  \\
     Buffer Size & \multicolumn{3}{c}{200000}  \\
     Feature Extractor & \multicolumn{3}{c}{MLP with hidden size the same as the state dimension}  \\
    \bottomrule
    \end{tabular}
    \caption{Hyperparameter settings in our experiments.}
    \label{table:hyper}
\end{table}

\section{Another Decomposition Theorem for Return Disparity}
\label{app-sec:thm}

\begin{restatable}{theorem}{rewardDisparityUpperBoundOcc}
For $g\in \{0, 1\}$, given policies $\pi_0, \pi_1\in \Pi$ and assume there exists a witness function class $\mathcal{F} = \{f: \sspace \times \aaspace \to \RR \}$, such that the reward functions $r_g(\cdot, \cdot)\in\mathcal{F}$ for $\forall~s\in\sspace, a\in\aaspace$ and $g\in\{0, 1\}$, then the following holds: 

\begin{equation*}
 \retgap \leq \frac{1}{1-\gamma} \bigg( \|r_0 - r_1\|_\infty + d_{\mathcal{F}}\big( \rho^{\pi_0}(s, a),  \rho^{\pi_1}(s, a) \big) \bigg),\\
\end{equation*}
 \label{thm:occ}
\end{restatable}

\paragraph{Remark} We see that return disparity is upper bounded by two terms: the distance between group-wise reward functions and the discrepancy between occupancy measures of the two MDPs. Given any two MDPs, the distance between group-wise reward functions is constant. If we further assume the two MDPs share the same reward function (\ie, $r_0(s, a) = r_1(s, a) = r(s, a),~\forall s\in\sspace, a\in\aaspace$), then the upper bound is simplified as
\begin{equation}
    \nonumber
    \retgap \leq d_{\mathcal{F}}\big( \rho^{\pi_0}(s, a),  \rho^{\pi_1}(s, a) \big).
\end{equation}
In this case, Theorem~\ref{thm:occ} implies a sufficient condition to minimize return disparity is to find policies $\pi_0, \pi_1 \in \Pi$ that minimize the distance between induced occupancy measures in the two MDPs. 
In what follows, we first give the detailed proof of Theorem~\ref{thm:occ} and give the algorithm design inspired by Theorem~\ref{thm:occ} in the subsequent section. 

\begin{proof}
By definition, the return disparity is
\begin{equation}
    \nonumber
    \begin{aligned}
    \retgap=&  \big|\Exp_{s\in \mu_0}[v^{\pi}(s)] - \Exp_{s\in \mu_1}[v^{\pi}(s)] \big| \\
    =& \frac{1}{1-\gamma} \bigg| \sum_{s\in \sspace}\sum_{a\in \aaspace} r_0(s, a) \rho^{\pi_0}(s, a) - \sum_{s\in \sspace}\sum_{a\in \aaspace} r_1(s, a) \rho^{\pi_1}(s, a) \bigg| \\
    =& \frac{1}{1-\gamma} \bigg| \sum_{s\in \sspace}\sum_{a\in \aaspace} r_0(s, a) \rho^{\pi_0}(s, a) - r_1(s, a) \rho^{\pi_1}(s, a) \bigg|\\
    =& \frac{1}{1-\gamma} \bigg| \sum_{s\in \sspace}\sum_{a\in \aaspace} r_0(s, a)\rho^{\pi_0}(s, a) - r_1(s, a)\rho^{\pi_0}(s, a) + r_1(s, a)\rho^{\pi_0}(s, a) - r_1(s, a) \rho^{\pi_1}(s, a) \bigg|\\
    \leq& \frac{1}{1-\gamma} \bigg| \sum_{s\in \sspace}\sum_{a\in \aaspace} r_0(s, a)\rho^{\pi_0}(s, a) - r_1(s, a)\rho^{\pi_0}(s, a)\bigg| + \frac{1}{1-\gamma} \bigg| \sum_{s\in \sspace}\sum_{a\in \aaspace} r_1(s, a)\rho^{\pi_0}(s, a) - r_1(s, a) \rho^{\pi_1}(s, a) \bigg|,\\
    \end{aligned}
\end{equation}
where the first term is upper bounded by 
\begin{equation}
    \nonumber
    \begin{aligned}
    \bigg| \sum_{s\in \sspace}\sum_{a\in \aaspace} r_0(s, a)\rho^{\pi_0}(s, a) - r_1(s, a)\rho^{\pi_0}(s, a)\bigg| =& \bigg| \sum_{s\in \sspace}\sum_{a\in \aaspace} \big(r_0(s, a)- r_1(s, a) \big) \rho^{\pi_0}(s, a)\bigg| \\
    \leq& \max_{s, a} | r_0(s, a) - r_1(s, a) | \sum_{s\in \sspace}\sum_{a\in \aaspace} \rho^{\pi_0}(s, a) = \|r_0 - r_1\|_\infty,\\
    \end{aligned}
\end{equation}
and the second term is upper bounded by 
\begin{equation}
    \nonumber
    \begin{aligned}
     \bigg| \sum_{s\in \sspace}\sum_{a\in \aaspace} r_1(s, a)\rho^{\pi_0}(s, a) - r_1(s, a) \rho^{\pi_1}(s, a) \bigg| =& \bigg| \sum_{s\in \sspace}\sum_{a\in \aaspace} r_1(s, a) \big( \rho^{\pi_0}(s, a) - \rho^{\pi_1}(s, a) \big) \bigg| \\
    \leq& \sup_{r_1\in \mathcal{F}} \bigg| \sum_{s\in \sspace}\sum_{a\in \aaspace} r_1(s, a) \big( \rho^{\pi_0}(s, a) - \rho^{\pi_1}(s, a) \big) \bigg| \\
    =& d_{\mathcal{F}}\big( \rho^{\pi_0}(s, a),  \rho^{\pi_1}(s, a) \big).\\
    \end{aligned}
\end{equation}
By symmetry, we also have
\begin{equation}
    \nonumber
    \bigg| \sum_{s\in \sspace}\sum_{a\in \aaspace} r_0(s, a)\rho^{\pi_0}(s, a) - r_1(s, a) \rho^{\pi_1}(s, a) \bigg| \leq d_{\mathcal{F}}\big( \rho^{\pi_0}(s, a),  \rho^{\pi_1}(s, a) \big).
\end{equation}
Combining the above results then completes the proof.
\end{proof}

\section{Algorithms to Mitigate Return Disparity via Occupancy Measures Alignment}
\label{app-sec:alg}

\begin{algorithm}[!ht]
 			\caption{Algorithm to mitigate return disparity via occupancy measures alignment under the framework of~\citep{dulac2015deep}.}
			\label{alg:occ}
			\begin{algorithmic}[1]
				\STATE Initialize policies $\pi_{\phi_0}$, $\pi_{\phi_1}$, Q-functions $q_{\theta_1}$ and $q_{\theta_2}$, feature extractors $f_{\psi_0}$ and $f_{\psi_1}$, environment buffers $\mathbi{D}_0$ and $\mathbi{D}_1$
				\FOR{each iteration}
			        \FOR{each environment step}
			            \FOR{$g\in\{0, 1\}$}
			            \STATE Sample an action $a_g$ using policy $\pi_{\phi_g}$; add the sample $(s_g, a_g, s^\prime_g, r_g)$ to $\mathbi{D}_g$;
			            \ENDFOR
			        \ENDFOR
				    \FOR{each model update step}
				        \FOR{$g\in\{0, 1\}$}
			            \STATE Update policy $\pi_{\phi_g}$, Q-function $q_{\theta_g}$, and feature extractor $f_{\psi_g}$ following~\citep{dulac2015deep};
			            \ENDFOR
			        \ENDFOR
			        \FOR{each occupancy measures alignment c}
			        \STATE Update feature extractors $f_{\psi_0}$ and $f_{\psi_1}$ via occupancy measures alignment via IPM (Wasserstein-1 distance/MMD);
			        \ENDFOR
				\ENDFOR
			\end{algorithmic}
\end{algorithm}

Inspired by Theorem~\ref{thm:occ}, we introduce an occupancy measures alignment procedure when learning group policies for the two MDPs. We use the Wolpertinger policy~\citep{dulac2015deep} based on Deep Deterministic Policy Gradient (DDPG)~\citep{lillicrap2015continuous} as as our baseline backbone algorithm, which is an actor-critic approach designed for learning policy in large discrete action spaces. 
We give an outline of our algorithm in Algorithm~\ref{alg:occ} and readers can refer to \citep{lillicrap2015continuous} for more clarification of the Wolpertinger policy. The details of the occupancy measures alignment step are similar to those of state visitation distributional alignment step. The experimental results of the implementation of Algorithm~\ref{alg:occ} are shown in Figure~\ref{fig:ddpg-results}.

The overall results in Figure~\ref{fig:ddpg-results} are similar to those in Figure~\ref{fig:main-results}.
However, compared to the Algorithm~\ref{alg}, the training processes of Algorithm~\ref{alg:occ} are more unstable. Besides, it also incurs additional learning costs since we have to train one policy for each MDP instead of one policy for both MDPs.

\begin{figure*}[!htb]
\centering
\begin{subfigure}[b]{.235\linewidth}
  \centering
  \includegraphics[width=1.1\linewidth]{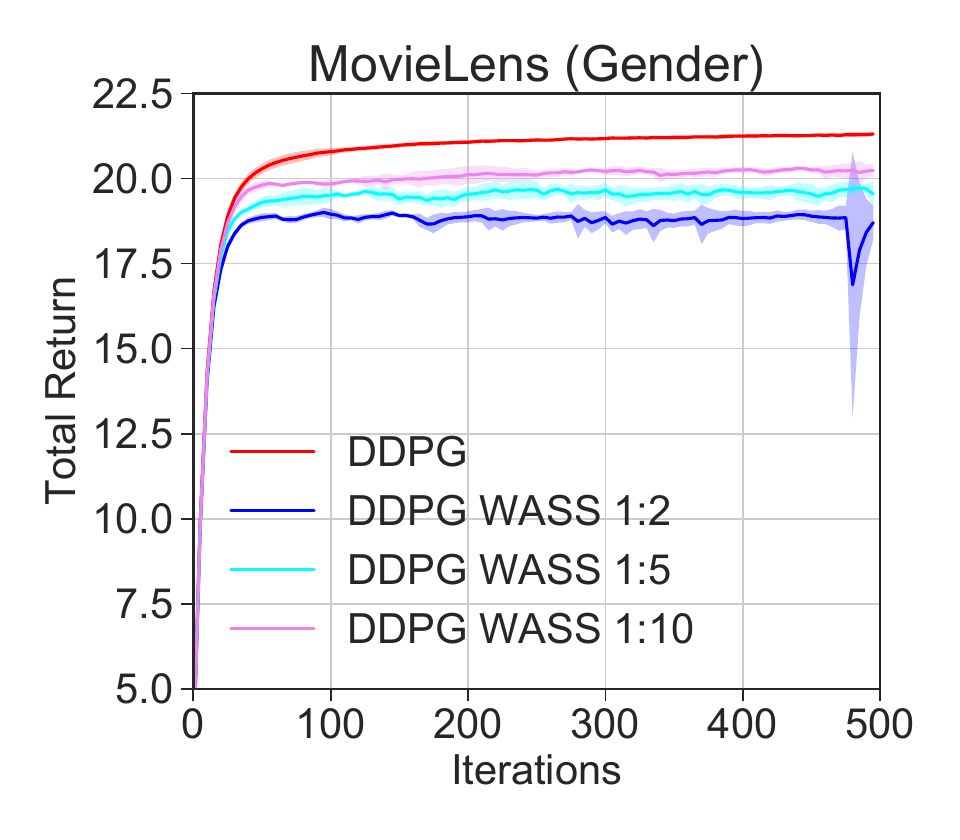}
\end{subfigure}
~
\begin{subfigure}[b]{.235\linewidth}
  \centering
\includegraphics[width=1.1\linewidth]{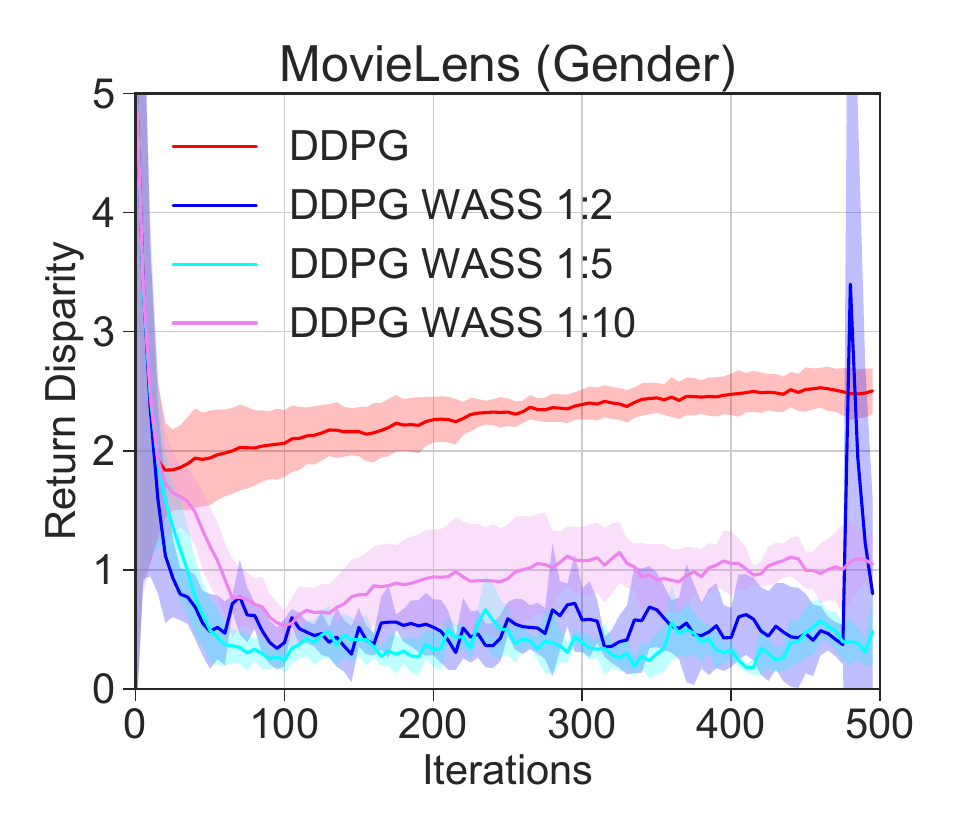}
\end{subfigure}
~
\begin{subfigure}[b]{.235\linewidth}
  \centering
  \includegraphics[width=1.1\linewidth]{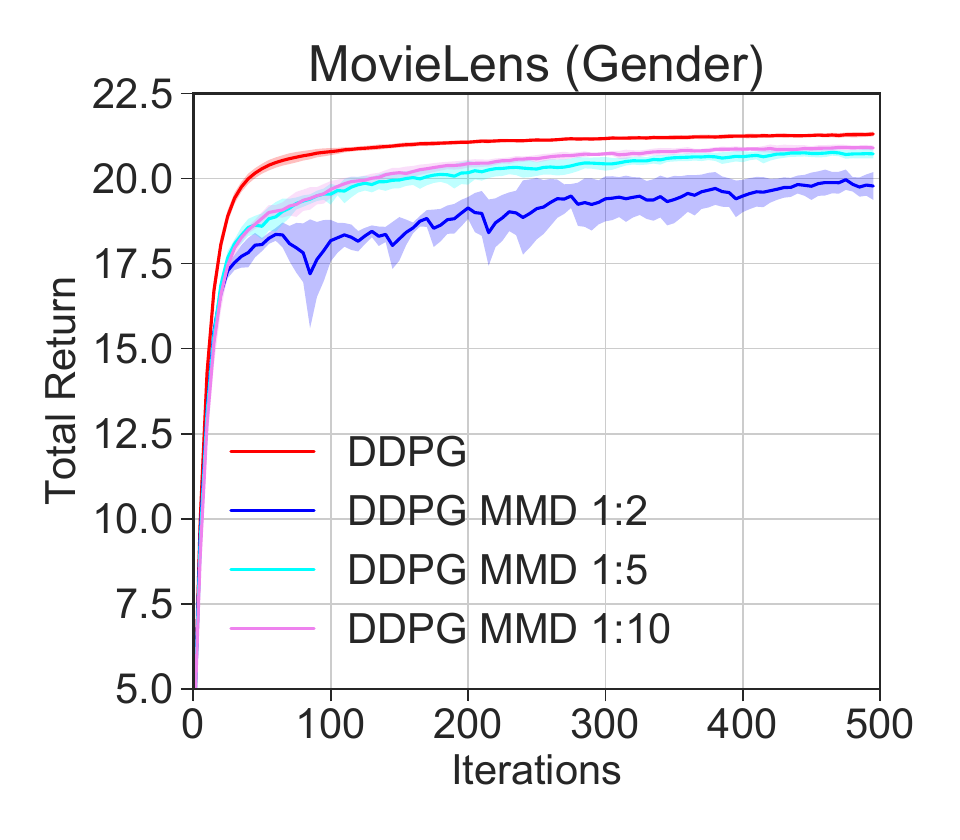}
\end{subfigure}
~
\begin{subfigure}[b]{.235\linewidth}
  \centering
  \includegraphics[width=1.1\linewidth]{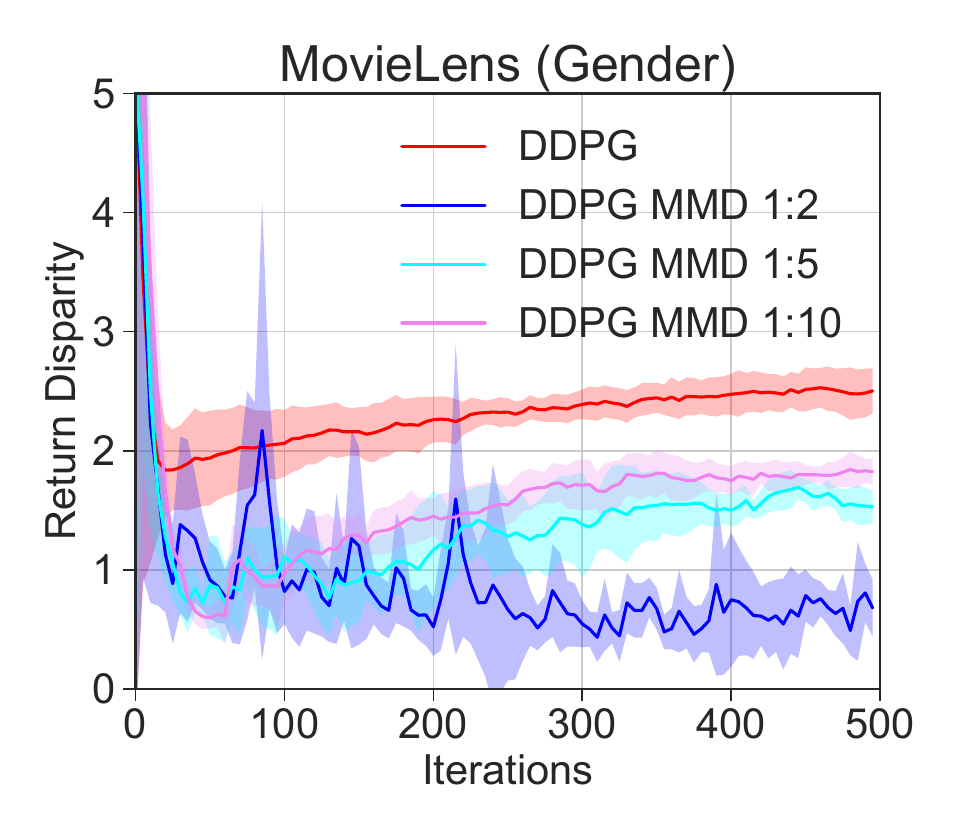}
\end{subfigure}
~
\begin{subfigure}[b]{.235\linewidth}
  \centering
  \includegraphics[width=1.1\linewidth]{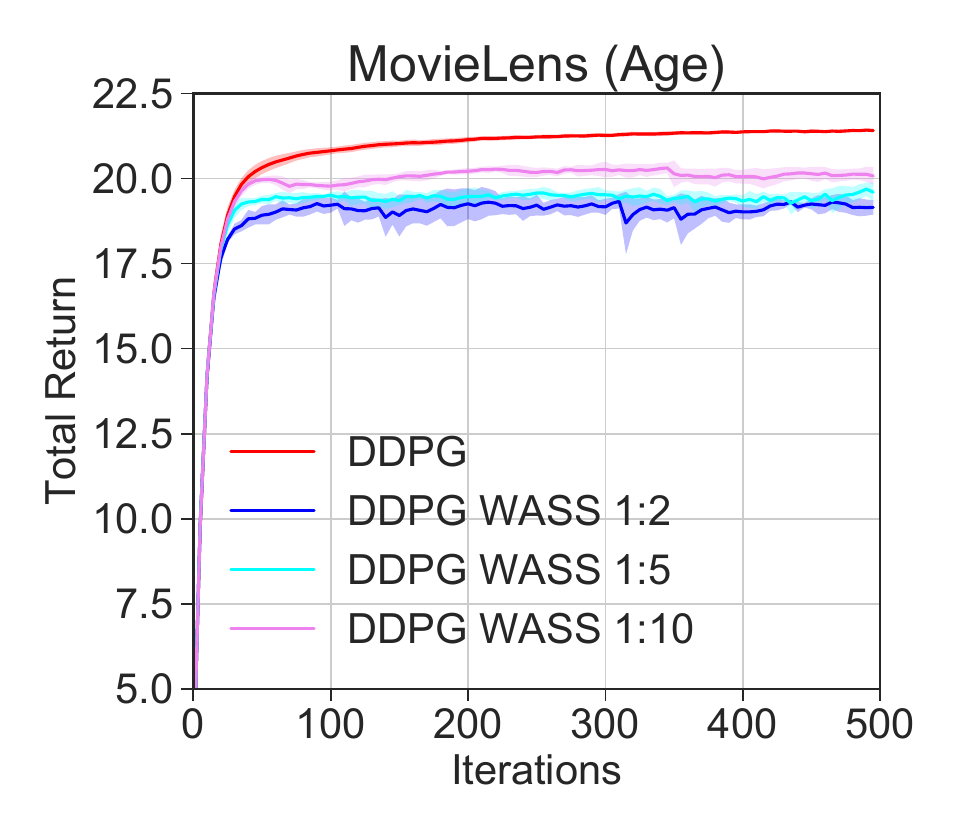}
\end{subfigure}
~
\begin{subfigure}[b]{.235\linewidth}
  \centering
\includegraphics[width=1.1\linewidth]{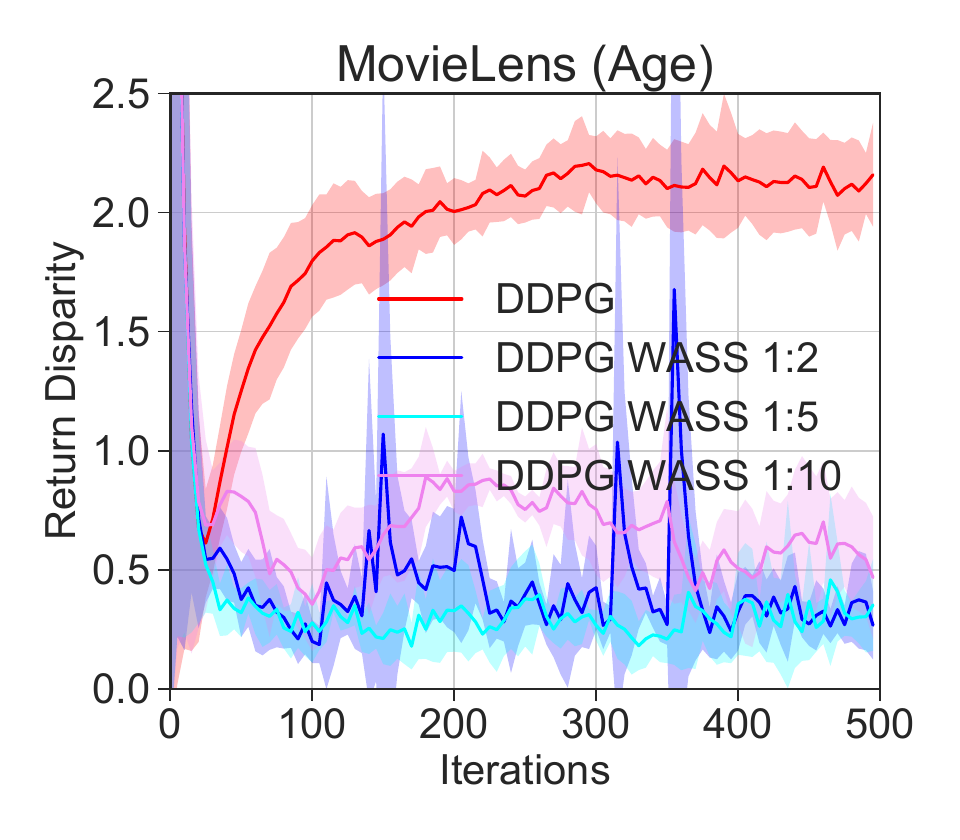}
\end{subfigure}
~
\begin{subfigure}[b]{.235\linewidth}
  \centering
  \includegraphics[width=1.1\linewidth]{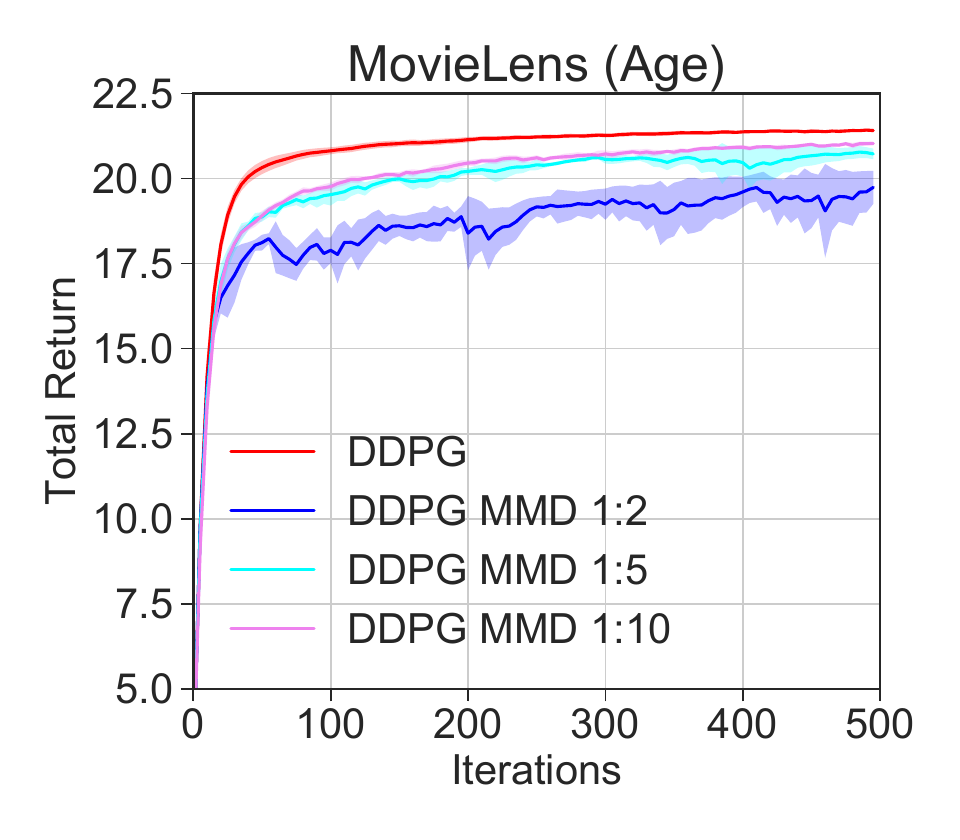}
\end{subfigure}
~
\begin{subfigure}[b]{.235\linewidth}
  \centering
  \includegraphics[width=1.1\linewidth]{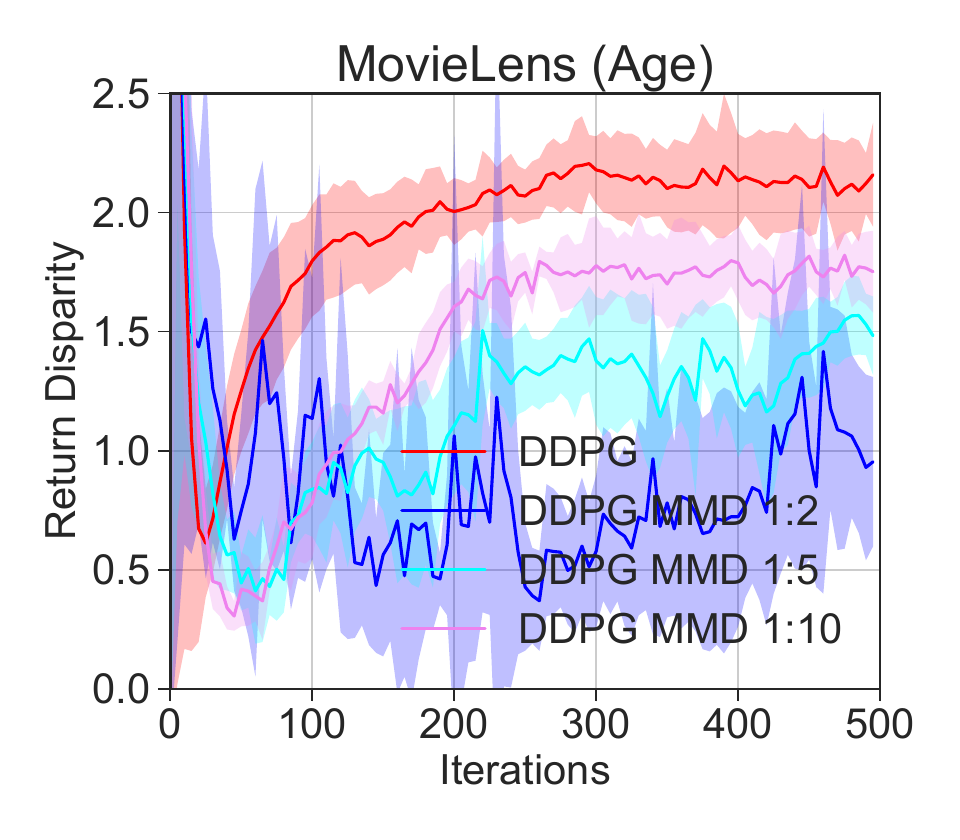}
\end{subfigure}
~
\begin{subfigure}[b]{.235\linewidth}
  \centering
  \includegraphics[width=1.1\linewidth]{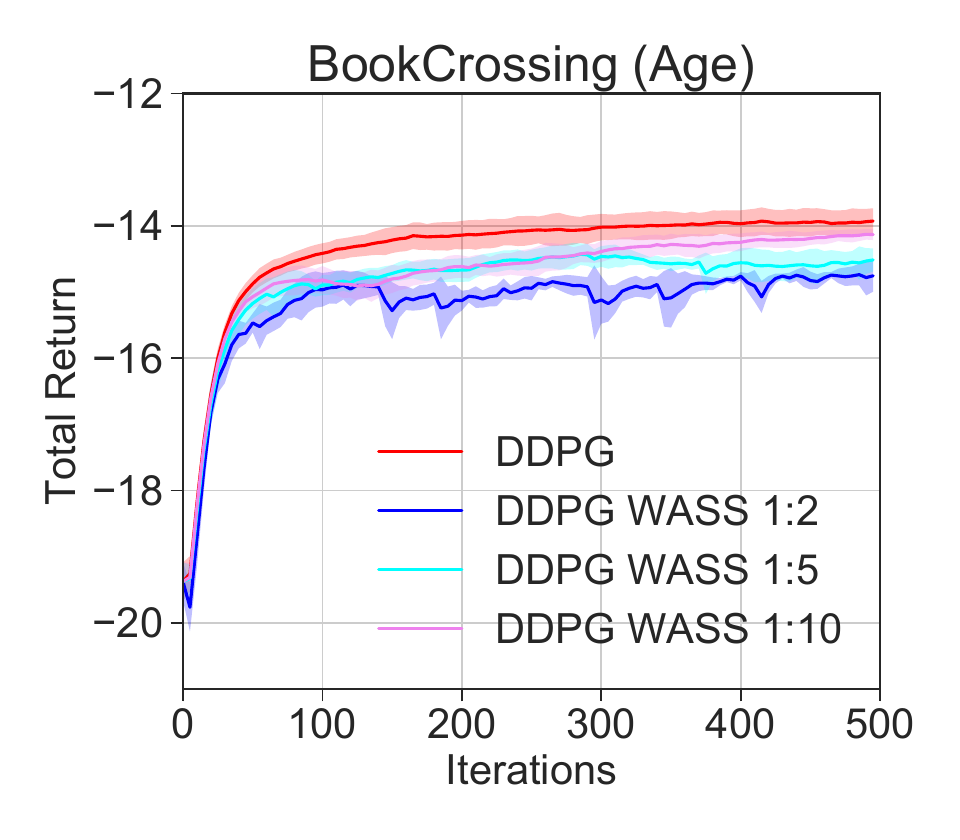}
\end{subfigure}
~
\begin{subfigure}[b]{.235\linewidth}
  \centering
\includegraphics[width=1.1\linewidth]{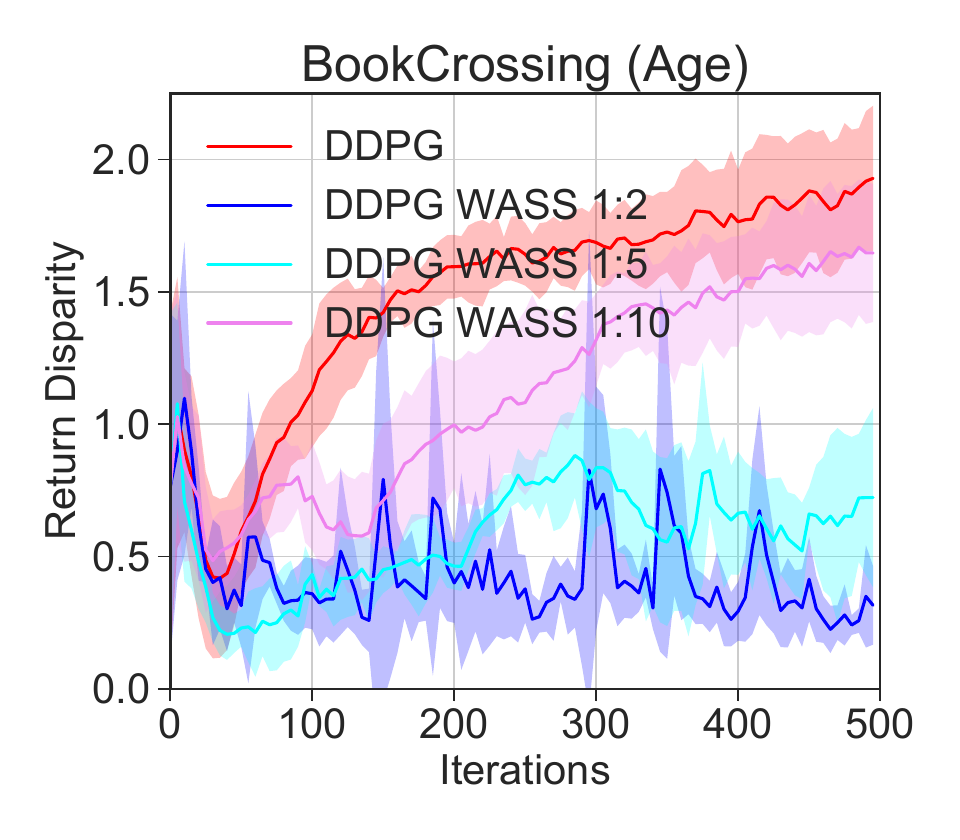}
\end{subfigure}
~
\begin{subfigure}[b]{.235\linewidth}
  \centering
  \includegraphics[width=1.1\linewidth]{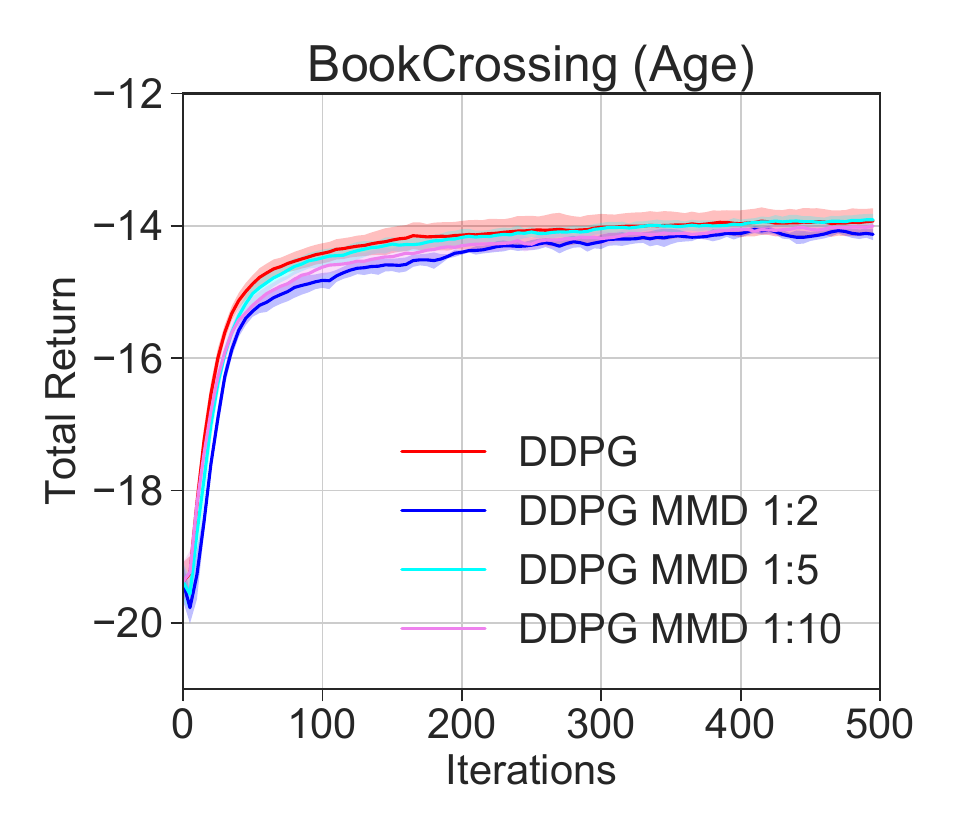}
\end{subfigure}
~
\begin{subfigure}[b]{.235\linewidth}
  \centering
  \includegraphics[width=1.1\linewidth]{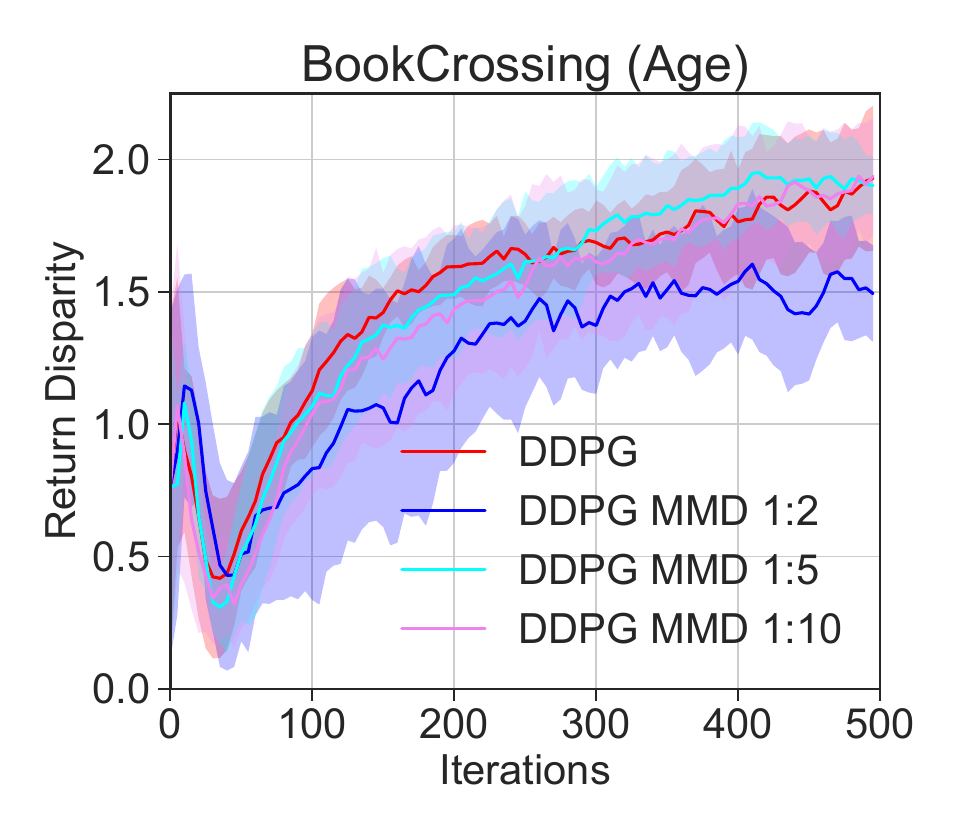}
\end{subfigure}
~
    \caption{Learning curves of DDPG, \textsc{DDPG-WASS} and \textsc{DDPG-MMD} in three different settings.}
    \label{fig:ddpg-results}
\end{figure*}

\end{document}